\RequirePackage{fix-cm}
\documentclass[smallextended]{svjour3}       
\smartqed  
\usepackage{amsmath}
\usepackage{amsfonts}
\usepackage{cite}
\usepackage{comment}
\usepackage{color}
\usepackage{amssymb}
\usepackage{mathrsfs}
\usepackage{graphicx}
\usepackage[font=scriptsize]{caption}
\usepackage[margin=3cm]{geometry}

\usepackage{kantlipsum} 
\UseRawInputEncoding

\usepackage{subfigure}

\usepackage[justification=centering]{caption}
 \usepackage{mathptmx}      
\begin{document}

\title{Radial Basis Function Approximation with Distributively  Stored Data on  Spheres}


\author{Han Feng \and Shao-Bo Lin\thanks{Corresponding Author:  sblin1983@gmail.com} \and Ding-Xuan Zhou
}


\institute{ H. Feng  \at  Department of Mathematics, City University of Hong Kong, Hong Kong      \\
             \and S. B. Lin \at Center for Intelligent Decision-Making and Machine Learning, School of Management, Xi'an Jiaotong University, Xi'an, China\\
             \and
              D. X. Zhou \at  School of Mathematics and Statistics, University of Sydney, Sydney NSW, Australia
}

\date{Received: date / Accepted: date}

\maketitle

\begin{abstract}
This paper proposes a distributed weighted regularized least squares algorithm (DWRLS) with radial basis functions  to tackle spherical data that are stored across numerous local servers and cannot be shared with each other. Via developing a novel integral operator approach based on spherical quadrature rules, we succeed in deriving optimal approximation rates for DWRLS and theoretically demonstrate that DWRLS performs similarly as running a weighted regularized least squares algorithm on the whole data stored  on a large enough machine.   This interesting finding implies that distributed learning is capable of sufficiently exploiting potential values of distributively stored spherical data, even though local servers cannot access   the whole data.

\end{abstract}

%


\section{Introduction}

\qquad In geophysics, solar system, climate prediction, environment
governance and meteorology, and image rendering, samples formed as input-output pairs  are collected over  {spheres}
\cite{Dolelson2003,Freeden1998,Tsai2006}, such as the surface of the earth and the direction of radiance. Due to the storage bottleneck and data privacy, these spherical data are often distributively stored across numerous computational servers. Typical examples include the CHAMP (Challenging Mini-satellite Payload ) data \cite{Reigber2002} that involve  billions of gravity and magnetic field measurements   and cannot be stored on a single sever, and the nuclear energy data \cite{Dittmar2012} that record the nuclear energy distribution for some countries and cannot be shared with others. The classical  fitting schemes such as spherical harmonics \cite{Muller1966}, spherical basis functions
\cite{Narcowich2007}, spherical wavelets \cite{Freeden1998},
spherical needlets \cite{Narcowich2006}, spherical kernel methods
\cite{Lin2019} and spherical filtered hyperinterpolation
\cite{Sloan1995} are incapable of tackling these distributively stored data since they require to access the whole data on a single serve.

Distributed learning \cite{Zhang2015}, based on a divide-and-conquer approach,
provides a promising  way to tackle distributively  stored spherical
data. This strategy applies a specific learning algorithm to a
data subset on each local server  to produce a  local estimator
(function), and then synthesizes a global estimator by utilizing some weighted
average of the obtained local estimators. Using some integral operator approaches, the feasibility of distributed learning has been verified in Euclidean spaces for distributed kernel ridge regression \cite{Zhang2015,Lin2017}, distributed kernel-based gradient descents \cite{Lin2018CA,Hu2020}, distributed kernel-based spectral algorithms \cite{Mucke2018,Linj2020} and distributed local average regression \cite{Chang2017} in the sense that distributed learning can achieve the optimal approximation rates of its batch counterpart, i.e., running corresponding algorithms on the whole data, provided the number of local servers is not so large and the samples are collected via a random manner. However, these interesting results do not apply to spherical data, mainly due to the fact that spherical data such as CHAMP and nuclear    energy data, gathered by satellites, are often sampled at fixed positions to save resources, making the existing analysis framework based on random sampling and concentration inequalities no more available. Furthermore, the spectrum of a kernel-based integral operator defined on  the  sphere,  one of the simplest example of  homogeneous manifolds, is   totally  different from that defined on Euclidean spaces, implying that the integral operator approach developed in \cite{Lin2017,Lin2018CA,Mucke2018,Linj2020,Hu2020,Wang2020} is infeasible for spherical data. In a word, there lacks a unified theoretical  analysis framework  to provide a springboard to understanding and designing distributed learning schemes for spherical data.

The purpose of this paper is to develop a distributed learning scheme to handle distributively stored spherical data and provide a theoretical analysis framework to verify its feasibility. Our study stems from  three interesting observations. At first, though   deterministic sampling on  spheres excludes the usability of concentration inequalities \cite{Massart2007} for random variables and makes the existing integral operator approach in \cite{Lin2017,Lin2018CA,Mucke2018,Linj2020,Hu2020}  infeasible, the well developed quadrature rules on spheres \cite{Mhaskar2001,Brown2005,Brauchart2007} that quantify the difference between integrals and their discretizations, provide an alternative  way to develop  an exclusive integral operator approach for spherical data.  Then, if a spherical quadrature rule is adopted to ease the theoretical analysis, the classical discrete (regularized) least-squares approaches on spheres \cite{Gia2007,Hesse2017} should be replaced by a weighted (regularized) least-squares scheme, just as \cite{Keiner2007} did for spherical harmonics approximation. Finally, once the integral operator theory is established,   the standard error decomposition technique  for distributed learning \cite{Chang2017JMLR,Guo2017} that connects the approximation errors of the global estimator  and local estimators is sufficient to derive the approximation error estimate of the distributed  learning scheme. Motivated by these observations, we develop a distributed weighted regularized least squares algorithm (DWRLS) associated with some spherical radial basis function to fit distributively stored spherical data.

Our main theoretical contributions are three folds. Firstly,   we succeed in developing an exclusive integral operator approach for deterministic sampling on  spheres, in which numerous bounds that describe   differences between an integral operator and its empirical counterpart are derived. The derived bounds for deterministic sampling are similar to those for random sampling in Euclidean spaces, without using any concentration inequalities in statistics. Secondly,  adopting the developed integral operator approach, we deduce optimal approximation rates of DWRLS, even when the data are heavily contaminated, that is, the noise of outputs is large.
Finally, we rigorously prove that DWRLS performs similarly as its batch counterpart in the sense that they achieve the same optimal order  of  approximation error, provided the number of local servers is not so large, showing the feasibility   of  distributed learning  to fit distributively stored spherical data.

The rest of the paper is organized as follows. In Section
\ref{Sec.algorithm}, we firstly  introduce spherical radial basis functions and then introduce the distributed weighted regularized least squares algorithm (DWRLS) on the sphere $\mathbb S^d$. In Section \ref{Sec:Main-Results}, we provide approximation error estimates for DWRLS. As a byproduct, we also derive optimal approximation error estimates for weighted regularized least squares algorithm (WRLS) with the whole data stored on a single serve. In Section \ref{Sec:operator-approach}, we develop a novel integral operator approach for spherical data based on spectrum analysis on the sphere and a spherical quadrature formula.
 Section \ref{Sec.proof} gives proofs of the main results. In Section \ref{sec.Simulation}, we conduct several numerical simulations to show the power of DWRLS in practice.

\section{Distributed Weighted Regularized  Least Squares on the Sphere}\label{Sec.algorithm}
In this  section, we propose a distributed weighted regularized  least squares (DWRLS) algorithm to tackle distributively stored scattered data on the unit sphere $\mathbb S^d$ of the $d+1$ Euclidean space $\mathbb R^{d+1}$. Assume that there are $m\in\mathbb N$ servers, each of which possesses a data set $D_j:=\{x_{i,j},y_{i,j}\}_{i=1}^{|D_j|}$ of cardinality $|D_j|$, where $\Lambda_j:=\{  x_{i,j}\}_{i=1}^{|D_j|}\subset\mathbb S^d$,
\begin{equation}\label{Model1:fixed}
        y_{i,j}=f^*(x_{i,j})+\varepsilon_{i,j},  \qquad\forall\
        i=1,\dots,|D_j|,j=1,\dots,m,
\end{equation}
$\{\varepsilon_{i,j}\}$ are independent  random  noise  satisfying $E[\varepsilon_{i,j}]=0$ and $|\varepsilon_{i,j}|\leq M$ with a constant $M>0$ and $f^*$ is a function to model the relation between the input $x$ and output $y$.
Since we do not impose strict restrictions on the magnitude of noise, the fitting problem in (\ref{Model1:fixed}) is different from that in \cite{Hesse2017} where the noise is assumed to be deterministic and extremely small. Without loss of generality, we assume further $D_{j}\cap D_{j'}=\varnothing$ for $j\neq j'$, implying that different severs possess different data.
Our aim is to design a  fitting scheme based on $D:=\cup_{j=1}^mD_j$, in the premise that the data in $D_j$ cannot be shared with each other,  to yield an estimator $\bar{f}_D$ such that $\bar{f}_D$ is near to $f^*$.

\subsection{Spherical basis function and native space}
For integer $k\geq0$, the restriction to $ \mathbb{S}^{d}$  of a
homogeneous harmonic polynomial of degree $k$  is
called a spherical harmonic of degree $k$. Denote by $\mathbb{H}^{d}_k$ and $\Pi_s^{d}$
the classes of all
spherical harmonics of degree $k$
and   all spherical polynomials of degree $k\leq s$, respectively. It can be found in \cite{Muller1966} that the dimensions of $\mathbb{H}^{d}_k$ and $\Pi_s^{d}$ are
$$
d_k^{d}:=\left\{\begin{array}{ll}
\frac{2k+d-1}{k+d-1}{{k+d-1}\choose{k}}, & \mbox{if}\ k\geq 1, \\
1, & \mbox{if}\ k=0
\end{array}
\right.
$$
with $d_k^{d}\sim k^{d-1}$ and $\sum_{k=0}^sd^{d}_k=d_s^{d+1}\sim s^d$, respectively. Throughout this paper, $a\sim b $ for $a,b\in\mathbb R$ means that there are absolute constants $\hat{c}_1,\hat{c}_2$ such that $\hat{c}_1b\leq a\leq\hat{c}_2b$.

Let $\{Y_{k,j}\}_{j=1}^{d_k^d}$ be  an arbitrary  orthonormal basis of
$\mathbb H_k^d$ and  $P_k^{d+1}$ be the normalized Legendre polynomial, i.e.,
  $P_k^{d+1}(1)=1$ and
$$
       \int_{-1}^1P_k^{d+1}(t)P^{d+1}_j(t)(1-t^2)^{\frac{d-2}2}dt
       =\frac{\Omega_{d}}{\Omega_{d-1}d_k^{d}}\delta_{k,j},
$$
where
$\Omega_d=\frac{2\pi^{\frac{d+1}{2}}}{\Gamma(\frac{d+1}{2})}$ denotes
  the volume of $\mathbb S^d$ and
 $\delta_{k,j}$ is the usual Kronecker symbol. It can be found in \cite{Szego1967} that
\begin{equation}\label{Legendre-pol-bound}
       |P_k^{d+1}(t)|\leq 1,\qquad \forall k\in\mathbb N, t\in[-1,1].
\end{equation}
The classical addition formula \cite{Muller1966} establishes a relation between  $Y_{k,\ell}$ and $P_k^{d+1}$ via
\begin{equation}\label{jiafadingli}
             \sum_{\ell=1}^{d_k^{d}}Y_{k,\ell}(x)Y_{k,\ell}(x')=\frac{d_k^{d}}{\Omega_{d}}P_k^{d+1}(x\cdot
                  x').
\end{equation}

We say that a  function $\phi\in L^2([-1,1])$ is a spherical basis function (SBF) if its expansion
$
           \phi(t)=\sum_{k=0}^\infty
            \hat{\phi}_k\frac{d_k^d}{\Omega_d} P_k^{d+1}(t)
$
has all Fourier-Legendre coefficients
$$
    \hat{\phi}_k:=  \Omega_{d-1}  \int_{-1}^1P_k^{d+1}(t)\phi(t)(1-t^2)^{\frac{d-2}2}dt>0.
$$
It is well known that each SBF $\phi$ corresponds to a native space $\mathcal N_\phi$ that is defined by
\begin{equation}\label{Native-space}
               \mathcal N_\phi:=\left\{f(x)=\sum_{k=0}^\infty\sum_{\ell=1}^{d_k^{d}}\hat{f}_{k,\ell}Y_{k,\ell}(x):
               \sum_{k=0}^\infty
                  \hat{\phi}_k^{-1}\sum_{\ell=1}^{d_k^{d}}|\hat{f}_{k,\ell}|^2<\infty\right\}
\end{equation}
 with inner product $
             \left\langle f,g\right\rangle_{\phi}:=\sum_{k=0}^\infty
              \hat{\phi}_k^{-1}\sum_{\ell=1}^{d_k^{d}}\hat{f}_{k,\ell}\hat{g}_{k,\ell}
$ and norm  $\|f\|_{\phi}:=\left(\sum_{k=0}^\infty
              \hat{\phi}_k^{-1}\sum_{\ell=1}^{d_k^{d}}|\hat{f}_{k,\ell}|^2\right)^{1/2}$,
where
$
                 \hat{f}_{k,\ell}:=\int_{\mathbb
                 S^d}f(x)Y_{k,\ell}(x)d\omega(x)
$
is the Fourier coefficient of $f$ with respect to $Y_{k,\ell}$ and $d\omega$ denotes the Lebesgue measure of the sphere.

If an SBF $\phi$ satisfies further $\sum_{k=0}^\infty\hat{\phi}_k\frac{d_k^d}{\Omega_d}<\infty$, then $\phi$ is said to be  (semi-)positive definite. Under this circumstance,
 $\mathcal N_\phi$ is a reproducing kernel
Hilbert space with  reproducing kernel
 $ (x, x') \mapsto \phi (x \cdot x')$.
Typical positive definite functions used for spherical data are the Sobolev-type functions with smoothness index $\gamma>d/2$,
\begin{equation}\label{sobolev-kernel}
       S_\gamma(t)=\sum_{k=0}^\infty(k(k+d-1)+1)^{-\gamma}\frac{d_k^d}{\Omega_d}P_k^{d+1}(t)
\end{equation}
 and the Gaussian function   with width $\tau>0$
\begin{equation}\label{Gaussian-kernel}
      G_\tau(t)=e^{-\frac{t^2}{\tau^2}}=\sum_{k=0}^\infty e^{-2/\tau^2} \tau^{d-1}\Gamma(d/2)I_{k+(d+1)/2-1}(2/\tau^2)\frac{d_k^d}{\Omega_d}P_k^{d+1}(t),
\end{equation}
where  $\Gamma(\cdot)$ is the Gamma function and $I_\nu(\cdot)$ is  modified Bessel function of the first kind, defined by
$$
     I_\nu(t)=\sum_{j=0}^\infty\frac{1}{j!\Gamma(\nu+j+1)}\left(\frac{t}{2}\right)^{v+2j}.
$$

%
%
%

%

\subsection{ Distributed weighted regularized least squares}
To introduce the distributed algorithm based on a positive definite kernels $\phi(x,x')=\phi(x\cdot x')$, we need spherical quadrature  rules for scattered data.
For an arbitrary $1\leq j\leq m$, define the mesh norm, separation radius and mesh ratio of $\Lambda_j$
by
$
                 h_{\Lambda_j}:=\max_{ x\in\mathbb S^d}\min_{ x_{i,j}\in \Lambda_j}\mbox{dist}( x,x_{i,j}),
$
$q_{\Lambda_j}:=\frac12\min_{i\neq i'}\mbox{dist}( x_{i,j}, x_{i',j})$ and $\rho_{\Lambda_j}:=\frac{h_{\Lambda_j}}{q_{\Lambda_j}}$, respectively. The mesh ratio $\rho_{\Lambda_j}\geq 1$ measures how uniformly the points of $\Lambda_j$ are distributed on
$\mathbb S^d$.
  We say  that $\Lambda_j$ is  $\tau_j$-quasi uniform, if there is a
constant $\tau_j\geq 2$ such that $\rho_{\Lambda_j}\leq \tau_j$. The existence of $\tau_j$-quasi uniform set with $\tau_j\geq 2$ has been   verified in \cite{Narcowich2007}.

For $s_j\in\mathbb N$,  a set $\mathcal Q_{\Lambda_j,s_j}:=\{(w_{i,j,s_j},  x_{i,j}): w_{i,j,s_j}\geq 0
\hbox{~and~}   x_{i,j}\in \Lambda_j\}$ is said to be    a positive
 quadrature rule   on $\mathbb S^d$ with degree $s_j\in\mathbb N$, if
\begin{equation}\label{eq:quadrature}
    \int_{\mathbb S^d}P(x)d\omega(x)=\sum_{x_{i,j}\in\Lambda_j} w_{i,j,s_j} P(  x_{i,j}), \qquad \forall P\in \Pi_{s_j}^d.
\end{equation}
The following  positive  quadrature rule  can be found in \cite[Theorem 3.1]{Brown2005} or
 \cite{Mhaskar2001}.

\begin{lemma}\label{Lemma:fixed cubature}
For every $1\leq j\leq m$, if $\Lambda_j=\{  x_{i,j}\}_{i=1}^{|D_j|}$ is $\tau_j$-quasi uniform and $s_j\leq c|D_j|^{1/d}$, then there exists a quadrature rule $\mathcal Q_{\Lambda_j,s_j}=\{(w_{i,j,s_j},  x_{i,j}): w_{i,j,s_j}\geq 0
\hbox{~and~}   x_{i,j}\in \Lambda_j\}$  satisfying  $0\leq w_{i,j,s_j}\leq c_1|D_j|^{-1}$,
where $c,c_1$ are constants  depending only on $\tau_j$ and
$d$.
\end{lemma}


With the help of the spherical quadrature rule, we can proceed our description as follows.
On the $j$-th server, we take an $s_j\in\mathbb N$ to admit a quadrature rule $\mathcal Q_{\Lambda_j,s_j}=\{(w_{i,j,s_j},  x_{i,j}): w_{i,j,s_j}\geq 0
\hbox{~and~}   x_{i,j}\in \Lambda_j\}$  with  $0\leq w_{i,j,s_j}\leq c_1|D_j|^{-1}$.
Then,   the  weighted regularized least squares (WRLS) on the $j$-th server, with a regularization parameter $\lambda_j>0$, is defined by
\begin{equation}\label{WRLS}
    f_{D_j,W_{j,s_j},\lambda_j}=\arg\min_{f\in\mathcal N_\phi}\sum_{(x_{i,j},y_{i,j})\in D_j}w_{i,j,s_j}(f(x_{i,j})-y_{i,j})^2+\lambda_j\|f\|_\phi^2.
\end{equation}
Finally, all these local estimators are transmitted to the global serve to synthesize a  global estimator, called as distributed weighted regularized least squares (DWRLS), as
\begin{equation}\label{DWRLS}
      \overline{f}_{D,W_{\vec{s}},\vec{\lambda}}=\sum\frac{|D_j|}{|D|}f_{D_j,W_{j,s_j},\lambda_j}.
\end{equation}

It should be mentioned that all  local estimators defined by (\ref{WRLS}) possess closed-form solutions.
Denote by $W_{D_j,s_j}$ the $|D_j|\times|D_j|$ diagonal matrix with diagonal elements $w_{1,s_1},\dots,w_{|D_j|,s_j}$ and $\Phi_{\Lambda_j}=\left(\phi(x_{i,j}\cdot x_{i',j})\right)_{i,i'=1}^{|D_j|}$. The following lemma, which can be easily derived from \cite[Thm 1.1.2]{Bjorck1996}, presents the analytic solution to the optimization problem defined by (\ref{WRLS}).
\begin{lemma}\label{Lemma:solution}
Let $f_{D_j,W_{j,s_j},\lambda_j}$ be defined by (\ref{WRLS}). Then $f_{D_j,W_{j,s_j},\lambda_j}=\sum_{(x_{i,j},y_{i,j})\in D_j}a_{i,j}\phi_{x_{i,j}}$ with
$$
       (a_{1,j},\dots,a_{|D|,j})^T=(W_{D_j,s_j}\Phi_{D_j}+\lambda_j I)^{-1}W_{D_j,s_j}y_{D_j},
$$
where $\phi_{x_{i,j}}=\phi(x_{i,j},\cdot)$ and $y_{D_j}=(y_{1,j},\dots,y_{|D_j|,j})^T$.
\end{lemma}


\section{Main Results}\label{Sec:Main-Results}
In this section, we derive approximation rates of DWRLS for noisy data (\ref{Model1:fixed}).

\subsection{Approximation capability of WRLS}
Before presenting the approximation rates of DWRLS, we should   provide a baseline for analysis, where the approximation error of WRLS is needed. Define
\begin{equation}\label{WRLS-whole}
    f_{D,W_{s},\lambda}=\arg\min_{f\in\mathcal N_\phi}\sum_{(x_{i},y_{i})\in D}w_{i,s}(f(x_{i})-y_{i})^2+\lambda\|f\|_\phi^2
\end{equation}
as the estimator derived by WRLS,
where $D=\{x_i\}_{i=1}^{|D|}$, $\Lambda:=\cup_{j=1}^m\Lambda_j$ and $\{w_{i,s}\}_{i=1}^{|D|}$  for $s\in\mathbb N$ is the quadrature weights of the quadrature rule
 $\mathcal Q_{\Lambda,s}=\{(w_{i,s},  x_{i}): w_{i,s}\geq 0
\hbox{~and~}   x_{i}\in \Lambda\}$ with $0\leq w_{i,s}\leq c_1|D|^{-1}$. It is easy to see that the WRLS estimator in (\ref{WRLS-whole}) is a batch version of DWRLS, which assumes that all data are stored on a single large server and WRLS is capable of handling them. According to Lemma \ref{Lemma:solution},  since the matrix-inversion is involved in WRLS, it requires $\mathcal O(|D|^2)$ memory requirements and $\mathcal O(|D|^3)$ float-computations to solve the optimization problem in (\ref{WRLS-whole}), which is
 infeasible when the data size is huge even if all  the data could be    collected without considering the data privacy issue. The study of approximation capability of WRLS (\ref{WRLS-whole}) is necessary, since it  enhances the understanding of DWRLS by means of determining which conditions are sufficient to guarantee that distributed learning performs similarly to its batch counterpart.

Let $\phi$ be a positive definite function with $0<\hat{\phi}_k<1$  for $k=1,2,\dots,$ and $\psi(t)=\sum_{k=0}^\infty
            \hat{\psi}_k\frac{d_k^d}{\Omega_d} P_k^{d+1}(t)$ be another SBF satisfying
\begin{equation}\label{kernel-relation}
     \hat{\psi}_k=\hat{\phi}_k^r, \qquad 0\leq r\leq 1.
\end{equation}
Therefore, we have  $\hat{\phi}_k\leq \hat{\psi}_k $  and consequently $N_\phi\subseteq N_\psi$. Our following theorem whose proof will be given in Section \ref{Sec.proof} presents an estimate for the approximation error of WRLS (\ref{WRLS-whole}) under the metric of $\mathcal N_\psi$.

 \begin{theorem}\label{Theorem:Generaliation-WRLS}
Let $0<\delta<1$,   $\Lambda$ be  $\tau$-quasi uniform and  $\mathcal Q_{\Lambda,s}:=\{(w_{i,s},  x_i): w_{i,s}\geq 0
\hbox{~and~}   x_i\in \Lambda\}$  be a quadrature rule satisfying $0<w_{i,s}\leq c_1|D|^{-1}$.
Under (\ref{Model1:fixed}) with $m=1$ and $f^*\in \mathcal N_\phi$ and (\ref{kernel-relation}) with $0\leq r\leq 1$, if  $\hat{\phi}_k\sim k^{-2\gamma}$ with $\gamma>d/2$,   $\lambda\sim|D|^{-\frac{2\gamma}{2\gamma +d}}$ and $s\geq\lambda^{-1/\gamma}$,  then with confidence $1-\delta$, there holds
\begin{equation}\label{optimal-rate-WRLS}
   \|f_{D,W_s,\lambda}-f^*\|_\psi\leq C|D|^{-\frac{(1-r)\gamma}{2\gamma+d}}\log\frac3\delta,
\end{equation}
where $C$ is a constant depending only on $c_1$, $\tau$, $d$, $M$ and $\|f^*\|_\phi$.
\end{theorem}

Setting $r=0$ in Theorem \ref{Theorem:Generaliation-WRLS}, we have $\mathcal N_\psi=L^2(\mathbb S^d)$. Then
 it follows from (\ref{optimal-rate-WRLS}) that
\begin{equation}\label{optimal-rate-WRLS-l2}
   \|f_{D,W_s,\lambda}-f^*\|_{L^2(\mathbb S^d)}\leq C|D|^{-\frac{\gamma}{2\gamma+d}}\log\frac3\delta
\end{equation}
holds with confidence $1-\delta$,
which coincides with the    approximation bounds derived in \cite{Lin2019,Lin2021} for random samples  and cannot be improved further according to the theory in \cite[Theorem 2]{Caponnetto2007}.
 Theorem \ref{Theorem:Generaliation-WRLS} requires $\lambda\sim|D|^{-\frac{2\gamma}{2\gamma +d}}$ and $s\geq\lambda^{-1/\gamma}$, implying $s\geq C_1|D|^{\frac{2}{2\gamma +d}}$ for some absolute constant $C_1$. Such a restriction is mild according to Lemma \ref{Lemma:fixed cubature}, from which we can deduce that   any $s\leq c|D|^{\frac{1}{d}}$  admits a positive quadrature formula as required. In this way, we can choose any $s$ satisfying $C_1|D|^{\frac{2}{2\gamma +d}} \leq s\leq c|D|^{\frac{1}{d}}$ in WRLS (\ref{WRLS-whole}) by noting $2\gamma>d$.

Scattered  data fitting on the sphere is a hot topic in approximation theory and numerical analysis. Numerous fitting schemes \cite{Fasshauer1998} have been developed for this purpose. The radial basis function approach has triggered enormous research activities \cite{Jetter1999,Narcowich20021,Levesley2005,Narcowich2007,Gia2007,Mhaskar2010,Hangelbroek2010,Hangelbroek2011,Hangelbroek2012}.
In particular, \cite{Gia2007} studied the  approximation capability of discrete least squares algorithm  and the approximation error of order $\mathcal O(|D|^{-\gamma/d})$ is derived for noiseless data; \cite{Narcowich2007} considered the native space barrier problem of SBF approximation and established a Bernstein-type inequality for shifts of SBF; \cite{Hangelbroek2011,Hangelbroek2012} studied the Lebesgue constant for kernel approximation on the sphere. It should be mentioned that these interesting results focused on interpolation problems, in which  the collected data are assumed to be clean, i.e., $\varepsilon_{i,j}=0$ in (\ref{Model1:fixed}). Differently, Theorem \ref{Theorem:Generaliation-WRLS} studies a fitting problem with noisy data, which imposes strict requirements on  the stability of the fitting scheme. The most related work in this direction is  \cite{Hesse2017}, where approximation error for kernel-based regularized least squares algorithm was derived for noisy spherical data. Our derived error in (\ref{optimal-rate-WRLS}) is a little bit worse than that of \cite{Hesse2017} at the first glance.
However,   it should be highlighted that we do no impose  the strict restrictionon that  the magnitude of noise is small. In fact, the noise to guarantee an approximation rate of order  $\mathcal O(|D|^{-\gamma/d})$ should be extremely small in \cite{Hesse2017}, while that to guarantee an approximation rate of order $\mathcal O(|D|^{-\gamma/(2\gamma+d)})$ can be very large. Furthermore, our result focuses on the Sobolev-type error estimate while that of \cite{Hesse2017} is only carried out in $L^2(\mathbb S^d)$. As mentioned above, the rate $\mathcal O(|D|^{-\gamma/(2\gamma+d)})$ cannot be improved further under the same setting as Theorem \ref{Theorem:Generaliation-WRLS}.  In short,
 different approximation rates between \cite{Hesse2017} and our work are due to the different types of data rather than the algorithm selection or proof skills.

Another line of related work is \cite{Caponnetto2007,Lin2017,Chang2017JMLR,Guo2020}, where  a similar order of approximation error was derived for similar kernel-based algorithms for random samples. There are two important differences between our work and the results in \cite{Caponnetto2007,Lin2017,Chang2017JMLR,Guo2020}. On one hand, we  consider  data  deterministically sampled over quasi-uniform points while the existing work focus on random samples. The deterministic setting   makes the widely used concentration inequalities \cite{Caponnetto2007} no more available for our analysis. Instead, we
develop a novel integral operator approach based on spherical quadrature rules. On the other hand, since our analysis involves spherical quadrature rules, we are interested in the weighted regularized least squares algorithm with quadrature weights rather than the regularized least squares in \cite{Caponnetto2007,Lin2017,Chang2017JMLR,Guo2020}. That is, our adopted algorithm is different from those in \cite{Caponnetto2007,Lin2017,Chang2017JMLR,Guo2020}.

It should be mentioned that in the existing literature we do not find any similar results as Theorem \ref{Theorem:Generaliation-WRLS} under the same setting of this paper,  though the statement   is a little bit standard. The main reason is that we are concerned with scattered  data fitting with large noise that is different from the classical scattered data fitting problem \cite{Narcowich20021,Narcowich2007,Hesse2017}  and deterministic samples which is different from the kernel learning  problem in statistical learning theory \cite{Caponnetto2007,Lin2017,Chang2017JMLR}.

\subsection{Approximation capability of DWRLS}

In the previous subsection, we present an optimal approximation error estimate of WRLS when the data satisfy (\ref{Model1:fixed}) with $m=1$, showing that WRLS can provide a perfect estimator, assuming that the whole data can be gathered together on a single server  and the
size of data  is not so large. Facing distributively stored data, WRLS fails to achieve the approximation error bounds as (\ref{optimal-rate-WRLS})  since there are only $|D_j|$ data available to the $j$-th local server. In this part, we show that DRWLS can yield similar approximation error bounds, provided the number of local servers is not so large. The following theorem  presents our main result  of this paper.

 \begin{theorem}\label{Theorem:Generaliation-DWRLS}
Let $1\leq j\leq m$, $s_j\in\mathbb N$, $\Lambda_j=\{x_{i,j}\}_{i=1}^{|D_j|}$ be  $\tau_j$-quasi uniform and  $\mathcal Q_{\Lambda_j,s_j}:=\{(w_{i,j,s_j},  x_{i,j}): w_{i,j,s}\geq 0
\hbox{~and~}   x_{i,j}\in \Lambda_j\}$  be a quadrature rule satisfying $0<w_{i,j,s_j}\leq c_{1,j}|D_j|^{-1}$ for some $c_{1,j}$ depending only on $d$ and $\tau_j$.
Under (\ref{Model1:fixed}) with $f^*\in \mathcal N_\phi$ and (\ref{kernel-relation}) with $0\leq r\leq 1$, if  $\hat{\phi}_k\sim k^{-2\gamma}$ with $\gamma>d/2$,   $\lambda_j\sim|D|^{-\frac{2\gamma}{2\gamma +d}}$ and $s_j\geq\lambda_j^{-1/\gamma}$,
then
\begin{equation}\label{optimal-rate-DWRLS}
   E[\|\overline{f}_{D,W_s,\lambda}-f^*\|_\psi]
   \leq C'|D|^{-\frac{(1-r)\gamma}{2\gamma+d}},
\end{equation}
where $C'$ is a constant depending only on  $\gamma,r,\tau_j$, $d$, $M$ and $\|f^*\|_\phi$.
\end{theorem}

It seems that there is no  restriction on the number of local servers in Theorem \ref{Theorem:Generaliation-DWRLS}. However, Lemma \ref{Lemma:fixed cubature} shows that to admit a positive quadrature rule in local servers,  $s_j$ for every $1\leq j\leq m$ should satisfy $s_j\leq |D_j|^{1/d}$. But $s_j$ in Theorem \ref{Theorem:Generaliation-DWRLS} should satisfy
$$
           s_j\geq\lambda_j^{-1/\gamma}\geq C_2|D|^{ \frac{2}{2\gamma+d}},
$$
implying
$$
    C_2|D|^{ \frac{2}{2\gamma+d}}\leq |D_j|^{1/d},\qquad \forall j=1,2,\dots,m,
$$
where $C_2$ is an absolute  constant. If $|D_1|\sim|D_2|\sim\cdots\sim|D_m|$, the above inequality yields
\begin{equation}\label{cond-on-m}
    m\leq C_3 |D|^{\frac{2\gamma-d}{2\gamma+d}}
\end{equation}
for some absolute  constant $C_3$. This illustrates that to guarantee the conditions of Theorem \ref{Theorem:Generaliation-DWRLS}, the number of local servers should  not be so large, which alternately implies that $|D_j|$ should not be  small for every $j\in\{1,\dots,m\}$.

Comparing Theorem \ref{Theorem:Generaliation-DWRLS} with Theorem \ref{Theorem:Generaliation-WRLS}, we find that under (\ref{cond-on-m}) and $|D_1|\sim|D_2|\sim\cdots\sim|D_m|$, DWRLS performs similarly to WRLS in the sense that their approximation rates  are of the same order, showing the power of DWRLS to fit distributively stored spherical data. In Theorem \ref{Theorem:Generaliation-DWRLS}, we also find that     regularization parameters  in all local servers are of the same order and are similar to  that for WRLS in Theorem \ref{Theorem:Generaliation-WRLS}. It is practically difficult to determine such a perfect $\lambda$ under the distributed learning framework, since $\gamma$ is difficult to quantify in practice.  As shown in Theorem \ref{Theorem:Generaliation-WRLS},   the theoretically optimal regularization parameter satisfies $\lambda\sim |D|^{-1/(2r+s)}$ and it can be     realized by using the well known cross-validation approach \cite[Chap.7]{Gyorfi2002}.
However, in the distributed learning setting, there are only $|D_j|$ data in the $j$-th local server and we are capable of getting  a regularization $\lambda_j\sim |D_j|^{-2\gamma/(2\gamma+d)}$ via  cross-validation in the $j$-th server. Practically,  we can set $\hat{\lambda}_j= (\lambda_j)^{\log_{|D_j|}|D|}$. Then $\lambda_j\sim |D_j|^{-2\gamma/(2\gamma+d)}$ implies $\hat{\lambda}_j\sim|D|^{-2\gamma/(2\gamma+d)}$, which is theoretically optimal as shown in Theorem \ref{Theorem:Generaliation-DWRLS}.

A distributed filter  hyperinterpolation scheme has already been developed for distributively stored spherical data in our recent work \cite{Lin2021}, where   optimal approximation error estimates were derived under the random sampling setting. There are mainly four differences between Theorem \ref{Theorem:Generaliation-DWRLS} and results in \cite{Lin2021}. First, the learning schemes are different. In particular, we study radial basis function approximation based on   weighted regularized least squares   while \cite{Lin2021} considered spherical polynomial approximation that can be constructed directly.
Second, the types of data are different. To be detailed, we study deterministic and quasi-unform   sampling while \cite{Lin2021} focused on random sampling. Third, the measurements of error are different. Indeed, we consider Sobolev-type error estimates but \cite{Lin2021} conducted the analysis only in $L^2(\mathbb S^d)$.
 Finally, the analysis frameworks are different. In this paper, we develop a novel integral operator approach to analyze the feasibility of DWRLS while the analysis in \cite{Lin2021} follows from the standard concentration inequality approach in \cite{Zhang2015,Chang2017,Lin2017}.

\section{Integral Operator Approach Based on Spherical Positive Quadrature Rules}\label{Sec:operator-approach}

In this section, we propose an integral operator approach based on spherical positive quadrature rules to derive approximation error   of WRLS and DWRLS.

\subsection{Spectrum of spherical basis function and spherical quadrature formulas}
Let $\phi$ be an SBF.   Define the integral operator $L_\phi:L^2(\mathbb S^d)\rightarrow L^2(\mathbb S^d)$ by
$$
    L_\phi f(x):=\int_{\mathbb{S}^{d}}\phi(x\cdot x')f(x')d\omega(x'), \qquad f\in L^2(\mathbb S^d).
$$
The   Funk-Hecke formula \cite{Muller1966}
\begin{equation}\label{funkhecke}
                   L_\phi Y_{k,j} (x)=\hat{\phi}_kY_{k,j}(x),\quad\forall\ j=1,\dots,d_k^d,k=0,1,\dots
\end{equation}
shows   that the eigen-pairs of $L_\phi$ are
\begin{equation}\label{eigen-pairs}
   (\hat{\phi}_0,Y_{0,1}), (\hat{\phi}_1,Y_{1,1}),\dots,(\hat{\phi}_1,Y_{1,d_1^d}),\dots, (\hat{\phi}_k,Y_{k,1}),\dots,(\hat{\phi}_k, Y_{d_k^d}),\dots.
\end{equation}
This implies
$$
   L_\phi f(x)= \sum_{k=0}^\infty\sum_{\ell=1}^{d_k^{d}}\hat{f}_{k,\ell}L_\phi Y_{k,\ell}(x)
   =\sum_{k=0}^\infty\hat{\phi}_k\sum_{\ell=1}^{d_k^{d}}\hat{f}_{k,\ell} Y_{k,\ell}(x), \qquad f\in L^2(\mathbb S^d)
$$
and
\begin{equation}\label{norm-operator-2}
    \|f\|_{\psi}^2= \sum_{k=0}^\infty\hat{\psi}_k^{-1}
               \sum_{\ell=1}^{d_k^{d}}|\hat{f}_{k,\ell}|^2
               =
                \sum_{k=0}^\infty \hat{\phi}_k^{-1}
               \sum_{\ell=1}^{d_k^{d}}\left| \hat{\phi}_k ^{(1-r)/2}\hat{f}_{k,\ell}\right|^2
               =
               \|L_\phi^{(1-r)/2}f\|_\phi^2,
\end{equation}
where $
           \psi(t)
$
is another SBF with Fourier-Legendre coefficients satisfying (\ref{kernel-relation}),
 and $\eta(  L_\phi)$ is defined  by spectrum calculus, i.e.,
$$
   \eta(  L_\phi) f(x)
   =\sum_{k=0}^\infty \eta(\hat{\phi}_k)\sum_{\ell=1}^{d_k^{d}}\hat{f}_{k,\ell} Y_{k,\ell}(x).
$$
 If $r=0$ in  (\ref{kernel-relation}), then  (\ref{norm-operator-2}) implies
 $\|f\|_{L^2(\mathbb S^d)}= \|L_\phi^{1/2}f\|_\phi$.

%

Let $\mathcal L_\phi$ be the integral operator on $\mathcal N_\phi$ defined  \cite{Smale2007} by
$$
    \mathcal L_\phi f(x):=\int_{\mathbb{S}^{d}}\phi(x\cdot x')f(x')d\omega(x'), \qquad f\in \mathcal N_\phi.
$$
Then, it follows from (\ref{Native-space}) and (\ref{eigen-pairs})  that the eigen-pairs of $\mathcal L_\phi$ are
\begin{equation}\label{eigen-pairs-Native}
   (\hat{\phi}_0,\sqrt{\hat{\phi}_0}Y_{0,1}), (\hat{\phi}_1,\sqrt{\hat{\phi}_1}Y_{1,1}),\dots, (\hat{\phi}_k,\sqrt{\hat{\phi}_k}Y_{k,1}),\dots,(\hat{\phi}_k, \sqrt{\hat{\phi}_k}Y_{d_k^d}),\dots.
\end{equation}
Therefore, for any $f\in \mathcal N_\phi$, it follows from (\ref{Native-space}) and (\ref{funkhecke}) that
\begin{eqnarray*}
    \mathcal L_\phi f
    &=& \mathcal L_\phi \sum_{k=0}^\infty
               \sum_{\ell=1}^{d_k^{d}}
               \langle f,\sqrt{\hat{\phi}_k}Y_{k,\ell}\rangle_\phi \sqrt{\hat{\phi}_k}Y_{k,\ell}
     =
    \sum_{k=0}^\infty
               \sum_{\ell=1}^{d_k^{d}}
              (\hat{\phi}_k)^{-1/2} \hat{f}_{k,\ell}   \mathcal L_\phi \sqrt{\hat{\phi}_k}Y_{k,\ell}\\
    &=&\sum_{k=0}^\infty
               \sum_{\ell=1}^{d_k^{d}}\hat{f}_{k,\ell}  \mathcal L_\phi Y_{k,\ell}
    =
    \sum_{k=0}^\infty\hat{\phi}_k\sum_{\ell=1}^{d_k^{d}}\hat{f}_{k,\ell} Y_{k,\ell}
    =L_\phi f.
\end{eqnarray*}
This implies
\begin{equation}\label{operator-equivalent}
    \eta(\mathcal L_\phi) f=\eta(L_\phi) f,\qquad \forall f\in \mathcal N_\phi,
\end{equation}
and
\begin{eqnarray}\label{operator-linear}
     \eta(\mathcal L_\phi) (f+g)
     &= &
    \sum_{k=0}^\infty\eta(\hat{\phi}_k)\sum_{\ell=1}^{d_k^d}\widehat{(f+g)}_{k,\ell}Y_{k,\ell}
    =
    \sum_{k=0}^\infty\eta(\hat{\phi}_k)\sum_{\ell=1}^{d_k^d}(f_{k,\ell}+g_{k,\ell})Y_{k,\ell} \nonumber\\
    &=&
    \eta(\mathcal L_\phi)f+\eta(\mathcal L_\phi)g, \qquad \forall f,g\in \mathcal N_\phi.
\end{eqnarray}

%



We  then deduce a   quadrature rule  which will play a crucial role in our integral operator approach in the following proposition.

\begin{proposition}\label{Proposition:3}
Let $\Xi=\{x_i\}_{i=1}^{|\Xi|}\subset \mathbb S^d$ be a set of scattered data  and  $\mathcal Q_{\Xi,s}:=\{(w_{i,s},  x_i): w_{i,s}> 0
\hbox{~and~}   x_i\in \Xi\}$  be  a positive
 quadrature rule   on $\mathbb S^d$ with degree $s\in\mathbb N$.
If $\hat \phi_k\sim k^{-2\gamma}$ with $\gamma> d/2$ and $\eta_{\lambda,u}(t)=(t+\lambda)^{-u}$, then for any $f, g \in \mathcal N_\phi$, any $u,v\in[0,1)$ satisfying $u+v\leq 1$ and any $\lambda\geq s^{-2\gamma}$, there holds
\begin{eqnarray*}
&&
   \left|\int_{\mathbb{S}^{d}}  (\eta_{\lambda,u}( L_\phi)f)  (x)(\eta_{\lambda,v}( L_\phi)g)  (x) d \omega(x)-\sum_{x_i\in\Lambda} w_{i, s} \left[(\eta_{\lambda,u}( L_\phi)f)  (x_{i} )(\eta_{\lambda,v}( L_\phi)g) (x_i )\right]\right|\nonumber\\
   &\leq&  \tilde{c}\|f\|_\phi\|g\|_\phi
   \left\{\begin{array}{cc}
   \lambda^{-u}s^{-\gamma}+\lambda^{-v}s^{-\gamma}+\lambda^{-u-v}s^{-2\gamma}+s^{-(1-u-v)\gamma},     & \mbox{if}\ u+v<1, \\
     \lambda^{-u}s^{-\gamma}+\lambda^{-v}s^{-\gamma}+\lambda^{-u-v}s^{-2\gamma}+\log s,   & \mbox{if}\ u+v=1,
   \end{array}
   \right.
\end{eqnarray*}
where $\tilde{c}$ is a constant depending only on
$\gamma$, $u,v$ and $d$.
\end{proposition}

The proof of Proposition \ref{Proposition:3} is a little bit standard, requiring the usage of the aforementioned spectrum analysis and spherical polynomials approximation \cite{Dai2006,Dai2006a}. We move it to Appendix for the sake of brevity. If we set $u=v=0$ and $g(x)\equiv1$, Proposition \ref{Proposition:3} is the classical spherical positive quadrature rule for Sobolev space established in \cite{Brauchart2007}.

\subsection{Operator representation and operator differences}
Let $\phi$ be a positive definite function.
Define  $S_Df=(f(x_1),\dots,f(x_{|D|}))^T$ and
\begin{equation}\label{def.adjoint-operator}
    S^T_{D,W_s}{\bf c}= \sum_{i=1}^{|D|}w_{i,s}c_i\phi_{x_i}.
\end{equation}
Write
\begin{equation}\label{def.empirical-operator}
     L_{\phi,D,W_s}f:=S^T_{D,W_s}S_Df=\sum_{i=1}^{|D|}w_{i,s}f(x_i)\phi_{x_i}.
\end{equation}
Then it is easy to check that $L_{\phi,D,W_s}:\mathcal N_\phi\rightarrow\mathcal N_\phi$ is a positive operator of finite rank. The following proposition  presents the operator representation of $f_{D,\lambda,W_s}$.
\begin{proposition}\label{Proposition:operator-KRR}
Let $f_{D,W_s,\lambda}$ be defined by (\ref{WRLS}) with $D_j$ being replaced by $D$. Then, $L_{\phi,D,W_s}:\mathcal N_\phi\rightarrow\mathcal N_\phi$ is a positive operator and
\begin{equation}\label{operator-KRR}
      f_{D,W_s,\lambda}=(L_{\phi,D,W_s}+\lambda I)^{-1}S^T_{D,W_s}y_D.
\end{equation}
\end{proposition}

Proposition \ref{Proposition:operator-KRR} is a standard result in kernel-based learning \cite{Smale2005,Smale2007}. We provide its proof in Appendix for the sake of completeness.
We then derive tight bounds for differences between the operator $\mathcal L_{\phi}$ and its empirical counterpart $L_{\phi,D,W_s}$, which is the core in our analysis. Our first result concerns an upper bound of $\|\mathcal L_{\phi}-L_{\phi,D,W_s}\|$, where $\|A\|$ denotes the spectral norm of the operator $A$.

\begin{proposition}\label{Proposition:operator-difference}
Let $\mathcal Q_{\Lambda,s}:=\{(w_{i,s},  x_i): w_{i,s}\geq 0
\hbox{~and~}   x_i\in \Lambda\}$  be  a positive
 quadrature rule   on $\mathbb S^d$ with degree $s\in\mathbb N$.
If $\hat \phi_k\sim k^{-2\gamma}$  with $\gamma> d/2$,
then for any $0\leq v< 1$ and $0<\lambda<1$, there holds
$$
  \|(\mathcal L_\phi+\lambda I )^{-v}(L_{\phi,D,W_s}-\mathcal L_\phi)\|\leq \tilde{c}(\lambda^{-v}s^{-\gamma}+ s^{-(1-v)\gamma}).
$$
\end{proposition}
\begin{proof}
Due to the definition  of operator norm and $(\mathcal L_\phi+\lambda I )^{-v} g\in \mathcal N_\phi$, we have
\begin{eqnarray*}
  &&\|(\mathcal L_\phi+\lambda I)^{-v}(L_{\phi,D,W_s}-\mathcal L_\phi)\|=\sup_{\|f\|_\phi\leq 1}\|(\mathcal L_\phi+\lambda I)^{-v}(L_{\phi,D,W_s}-\mathcal L_\phi)f\|_\phi\\
  &=&
  \sup_{\|f\|_\phi\leq 1}\sup_{\|g\|_\phi\leq 1}
  \langle (\mathcal L_\phi+\lambda I)^{-v}(L_{\phi,D,W_s}-\mathcal L_\phi)f,g\rangle_\phi\\
  &=&
  \sup_{\|f\|_\phi\leq 1}\sup_{\|g\|_\phi\leq 1}\langle (L_{\phi,D,W_s}-\mathcal L_\phi)f,
  (\mathcal L_\phi+\lambda I)^{-v}g\rangle_\phi\\
  &=&
   \sup_{\|g\|_\phi\leq 1,\|f\|_\phi\leq 1}
   \left|\left\langle\int_{\mathbb S^d}f(x')\phi_{x'}d\omega(x')-\sum_{i=1}^{|D|}w_{i,s}f(x_i)\phi_{x_i}, (\mathcal L_\phi+\lambda  I )^{-v}g\right\rangle_\phi\right|\\
   &=&
   \sup_{\|g\|_\phi\leq 1,\|f\|_\phi\leq 1} \left|\int_{\mathbb S^d}f(x')\langle\phi_{x'},(\mathcal L_\phi+\lambda I)^{-v} g\rangle_\phi d\omega(x')-\sum_{i=1}^{|D|}w_{i,s}f(x_i)\langle\phi_{x_i}, (\mathcal L_\phi+\lambda I)^{-v}g\rangle_\phi\right|\\
   &=&
   \sup_{\|g\|_\phi\leq 1,\|f\|_\phi\leq 1}\left|\int_{\mathbb S^d}f(x')(\mathcal L_\phi+\lambda I )^{-v}g(x')d\omega(x')-\sum_{i=1}^{|D|}w_{i,s}f(x_i)(\mathcal L_\phi+\lambda I)^{-v}g(x_i)\right|.
\end{eqnarray*}
Then, it  follows from Proposition \ref{Proposition:3} with $u=0$  that
\begin{eqnarray*}
   \|(\mathcal L_\phi+\lambda I )^{-v}(L_{\phi,D,W_s}-\mathcal L_\phi)\|\leq \tilde{c}(\lambda^{-v}s^{-\gamma}+ s^{-(1-v)\gamma}).
\end{eqnarray*}
This completes the proof of Proposition \ref{Proposition:operator-difference}.
\end{proof}

Our next tool  concerns   bounds of operator products.

\begin{proposition}\label{Proposition:Product}
Let $0\leq v\leq 1/2$, $0<\lambda<1$ and $\mathcal Q_{\Lambda,s}:=\{(w_{i,s},  x_i): w_{i,s}\geq 0
\hbox{~and~}   x_i\in \Lambda\}$  be  a positive
 quadrature rule   on $\mathbb S^d$ with degree $s\in\mathbb N$.
If $\hat \phi_k\sim k^{-2\gamma}$ with $\gamma> d/2$,
then
\begin{equation}\label{product-1}
   \|(\mathcal L_\phi+\lambda I)^{-1}(L_{\phi,D,W_s}+\lambda I)\|
   \leq
   \tilde{c}(\lambda^{-1}s^{-\gamma}+\lambda^{-1+v}s^{-(1-v)\gamma})+1
\end{equation}
and
\begin{equation}\label{product-2}
     \|(L_{\phi,D,W_s}+\lambda I)^{-1}(\mathcal L_{\phi}+\lambda I)\|
    \leq 2\tilde{c}^2  \lambda^{-2+2v} s^{-2(1-v)\gamma}+2\tilde{c}^2\lambda^{-2}s^{-2\gamma}+
    \tilde{c}\lambda^{-1+v}s^{-(1-v)\gamma}+\tilde{c}\lambda^{-1}s^{-\gamma}+1.
\end{equation}
\end{proposition}
\begin{proof}
For positive operators $A,B$, since
$$
       A^{-1}B=(A^{-1}-B^{-1})B+I=A^{-1}(B-A)+I,
$$
we have
$$
     (\mathcal L_\phi+\lambda I)^{-1}(L_{\phi,D,W_s}+\lambda I)
      =
     (\mathcal L_\phi+\lambda I)^{-1}(L_{\phi,D,W_s}-\mathcal L_\phi)+I.
$$
Then it follows Proposition \ref{Proposition:operator-difference} that
\begin{eqnarray*}
   &&\|(\mathcal L_\phi+\lambda I)^{-1}(L_{\phi,D,W_s}+\lambda I)\|
   \leq
   \|(\mathcal L_\phi+\lambda I)^{-1}(L_{\phi,D,W_s}-\mathcal L_\phi)\|+1\\
   &\leq&
   \|(\mathcal L_\phi+\lambda I)^{-1+v}\| \|(\mathcal L_\phi+\lambda I)^{-v}(L_{\phi,D,W_s}-\mathcal L_\phi)\|+1
   \leq
   \tilde{c}\lambda^{-1+v}(\lambda^{-v}s^{-\gamma}+ s^{-(1-v)\gamma})+1.
\end{eqnarray*}
This proves (\ref{product-1}). Concerning (\ref{product-2}), we use the second order decomposition developed in \cite{Lin2017}, i.e.,
$$
       A^{-1}B= A^{-1}(B-A)B^{-1}(B-A)+B^{-1}(B-A)+I
$$
and obtain
\begin{eqnarray*}
     (L_{\phi,D,W_s}+\lambda I)^{-1}(\mathcal L_{\phi}+\lambda I)
      &=&
     (L_{\phi,D,W_s}+\lambda I)^{-1}(\mathcal L_\phi-L_{\phi,D,W_s})(\mathcal L_{\phi}+\lambda I)^{-1} (\mathcal L_\phi-L_{\phi,D,W_s})\\
     &+&
     (\mathcal L_{\phi}+\lambda I)^{-1} (\mathcal L_\phi-L_{\phi,D,W_s})
     +I.
\end{eqnarray*}
Therefore, if $v\leq 1/2$, we have from (\ref{product-1}) and Proposition \ref{Proposition:operator-difference} that
$$
    \|(L_{\phi,D,W_s}+\lambda I)^{-1}(\mathcal L_{\phi}+\lambda I)\|
    \leq 2\tilde{c}^2  \lambda^{-2+2v} s^{-2(1-v)\gamma}+2\tilde{c}^2\lambda^{-2}s^{-2\gamma}+\tilde{c}\lambda^{-1+v}s^{-(1-v)\gamma}+\tilde{c}\lambda^{-1}s^{-\gamma}+1.
$$
This completes the proof of Proposition \ref{Proposition:Product}.
\end{proof}

Our next bound is  on the difference between $L_{\phi,D,W_s}f_\rho$ and $S_{D,W_s}^Ty_D$.
%
%
%
%
%

 \begin{proposition}\label{Proposition:value-difference-random}
Let $0<\delta<1$,    $\mathcal Q_{\Lambda,s}:=\{(w_{i,s},  x_i): w_{i,s}\geq 0
\hbox{~and~}   x_i\in \Lambda\}$  be  a quadratic rule on the sphere with $0\leq w_{i,s}\leq c_1|D|^{-1}$. If  $\hat{\phi}_k\sim k^{-2\gamma}$ with $\gamma>d/2$ and $y_i=f^*(x_i)+\varepsilon_i$ with $\varepsilon_i$ i.i.d. random noise satisfying $E[\varepsilon_i]=0$ and $|\varepsilon_i|\leq M$ for some $M>0$, then with confidence $1-\delta$, there holds
$$
    \left\|(\mathcal L_\phi+\lambda I)^{-1/2}(L_{\phi,D,W_s}f^*-S_{D,W_s}^Ty_D)\right\|_\phi
    \leq
    \tilde{c}'M \lambda^{-\frac{d}{4\gamma}} |D|^{-1/2}\log\frac3\delta,
$$
where $ \tilde{c}'$ is a constant depending only on $d$ and $c_1$.
\end{proposition}

To prove Proposition \ref{Proposition:value-difference-random}, we
need two tools. The first one is the well known  Hoeffiding lemma \cite{Massart2007}.

\begin{lemma}\label{Lemma:Hoeffding-1}
Let $X$ be a random variable with $  E[X]=0$, $a\leq X\leq b$. Then for $u>0$, there holds
$$
      E\left[e^{uX}\right]\leq e^{u^2(b-a)^2/8}.
$$
\end{lemma}

The other is a variant  of     Hoeffding's tail inequality.

\begin{lemma}\label{Lemma:Hoeffding}
Let $0\leq w_{i,s}\leq c_1|D|^{-1}$.
Then for any $t>0$ and $k=0,1,\dots$, we have
$$
    P\left[\sum_{i=1}^{|D|}\sum_{i'\neq i}w_{i,s}w_{i',s}(f^*(x_i)-y_i)(f^*(x_{i'})-y_{i'})P_{k}^{d+1}(x_i\cdot x_{i'})\geq t\right]
    \leq
     \exp\left( -2M^{ -4}c_1^{- 4}t^2|D|^2 \right).
$$
\end{lemma}

\begin{proof}
At first, we recall the well known Chernoff's inequality, showing that for any $t,b>0$ and random variable $\xi$, there holds
$$
       P\left[ \xi \geq t\right]  \leq  \frac{  E[e^{b\xi}]}{e^{bt}}.
$$
Set $\xi_k=\sum_{i=1}^{|D|}\sum_{i'\neq i}w_{i,s}w_{i',s}(f^*(x_i)-y_i)(f^*(x_{i'})-y_{i'})P_{k}^{d+1}(x_i\cdot x_{i'})$.
Since $\{\varepsilon_i\}$  are independent random variables and
$$
    E[y_i]=E[f^*(x_i)+\varepsilon_i]=f^*(x_i)+E[\varepsilon_i]=f^*(x_i),
$$
we have $E[\xi_k]=0$ for any $k=0,1,\dots$.
Then Chernoff's inequality implies
\begin{eqnarray*}
     &&P\left[\xi_k \geq t\right]\leq  e^{-bt} E\left[\exp\left(b\xi_k\right)\right]\\
    &=&
     e^{-bt}\prod_{i=1}^{|D|}\prod_{i'\neq i}  E\left[\exp\left({bw_{i,s}w_{i',s}(f^*(x_i)-y_i)(f^*(x_{i'})-y_{i'})P_{k}^{d+1}(x_i\cdot x_{i'})}\right)\right].
\end{eqnarray*}
But $0\leq w_i,w_{i'}\leq c_1|D|^{-1}$ and (\ref{Legendre-pol-bound}) yield
$$
    |w_{i,s}w_{i',s}(f^*(x_i)-y_i)(f^*(x_{i'})-y_{i'})P_{k}^{d+1}(x_i\cdot x_{i'})|\leq c_1^2M^2|D|^{-2}.
$$
Hence,
Lemma \ref{Lemma:Hoeffding-1} implies for any $i=1,\dots,|D|$,
$$
    E\left[\exp\left({bw_{i,s}w_{i',s}(f^*(x_i)-y_i)(f^*(x_{i'})-y_{i'})P_{k}^{d+1}(x_i\cdot x_{i'})}\right)\right] \leq \exp\left(M^4c_1^4b^2|D|^{-4}/2\right).
$$
Therefore, we obtain
\begin{eqnarray*}
     P\left[\xi_k \geq t\right]
    \leq
     e^{-bt}\prod_{i=1}^{|D|}\prod_{i' \neq i}\exp\left(M^4c_1^4b^2|D|^{-4}/2\right)
   \leq
     e^{-bt}e^{M^4c_1^4b^2|D|^{-2}/2}.
\end{eqnarray*}
  Setting $b= 2M^{-4}c_1^{-4}|D|^2t$, we obtain
$$
    P\left[\xi_k \geq t\right]
    \leq
      \exp\left({-\frac{2t^2|D|^2}{M^{ 4}c_1^{ 4}}}\right),\qquad\forall k=0,1,\dots.
$$
This completes the proof of Lemma \ref{Lemma:Hoeffding}.
\end{proof}

Based on the above  lemma, we can prove Proposition \ref{Proposition:value-difference-random} as follows.

\begin{proof}[Proof of Proposition \ref{Proposition:value-difference-random}]
Due to  definitions of $L_{\phi,D,W_s}$ and $S_{D,W_s}^T$, we have
$$
   L_{\phi,D,W_s}f^*-S_{D,W_s}^Ty_D
   =\sum_{i=1}^{|D|}w_{i,s}(f^*(x_i)-y_i)\phi_{x_i}.
$$
According to \eqref{operator-linear}, we then get
\begin{eqnarray}\label{p.1.1.1}
   &&\left\|(\mathcal L_\phi+\lambda I)^{-1/2}(L_{\phi,D,W_s}f^*-S_{D,W_s}^Ty_D)\right\|_\phi^2\\
   &=&
  \left\langle \sum_{i=1}^{|D|}w_{i,s}(f^*(x_i)-y_i)(\mathcal L_\phi+\lambda I)^{-1}\phi_{x_i},
   \sum_{i'=1}^{|D|}w_{i',s}(f^*(x_{i'})-y_{i'})\phi_{x_{i'}}\right\rangle_\phi \nonumber\\
   &=&
   \sum_{i=1}^{|D|}\sum_{i'=1}^{|D|}w_{i,s}w_{i',s}(f^*(x_i)-y_i)(f^*(x_{i'})-y_{i'})\langle (\mathcal L_\phi+\lambda I)^{-1}\phi_{x_i},\phi_{x_{i'}}\rangle_\phi  \nonumber\\
   &=&
   \sum_{i=1}^{|D|}\sum_{i'\neq i}w_{i,s}w_{i',s}(f^*(x_i)-y_i)(f^*(x_{i'})-y_{i'})\langle (\mathcal L_\phi+\lambda I)^{-1}\phi_{x_i},\phi_{x_{i'}}\rangle_\phi\nonumber\\
   &+&
   \sum_{i=1}^{|D|} w_{i,s}^2(f^*(x_i)-y_i)^2\langle (\mathcal L_\phi+\lambda I)^{-1}\phi_{x_i},\phi_{x_{i}}\rangle_\phi. \nonumber
\end{eqnarray}
But  $\phi(t)=\sum_{k=0}^\infty
            \hat{\phi}_k\frac{d_k^d}{\Omega_d} P_k^{d+1}(t)$ and (\ref{jiafadingli}) yield
$$
      \phi(x\cdot x')=\sum_{k=0}^\infty\hat{\phi}_k\sum_{\ell=1}^{d_k^d}Y_{k,\ell}(x)Y_{k,\ell}(x'),\qquad\forall x,x'\in\mathbb S^d.
$$
Then,
$$
   (\mathcal L_\phi+\lambda I)^{-1}\phi_{x}=\sum_{k=0}^\infty(\hat{\phi}_k+\lambda I)^{-1}\hat{\phi}_k\sum_{\ell=1}^{d_k^d}Y_{k,\ell}(x)Y_{k,\ell}(\cdot).
$$
Therefore, it follows from the reproducing property of $\phi$ and (\ref{jiafadingli}) that
\begin{eqnarray*}
    &&\langle (\mathcal L_\phi+\lambda I)^{-1}\phi_{x},\phi_{x'}\rangle_\phi
    =
    \left\langle \sum_{k=0}^\infty(\hat{\phi}_k+\lambda I)^{-1}\hat{\phi}_k\sum_{\ell=1}^{d_k^d}Y_{k,\ell}(x)Y_{k,\ell},\phi_{x'}\right\rangle_\phi\\
    &=&
     \sum_{k=0}^\infty(\hat{\phi}_k+\lambda I)^{-1}\hat{\phi}_k \sum_{\ell=1}^{d_k^d}Y_{k,\ell}(x')Y_{k,\ell}(x)
     =
   \sum_{k=0}^\infty(\hat{\phi}_k+\lambda I)^{-1}\hat{\phi}_k\frac{d_k^{d}}{\Omega_{d}}P_k^{d+1}(x\cdot
                  x').
\end{eqnarray*}
Plugging the above equations into (\ref{p.1.1.1}) and noting $P_{k}^{d+1}(1)=1$, we have
\begin{eqnarray*}
   &&\left\|(\mathcal L_\phi+\lambda I)^{-1/2}(L_{\phi,D,W_s}f^*-S_{D,W_s}^Ty_D)\right\|_\phi^2\\
   &=&
   \sum_{k=0}^\infty(\hat{\phi}_k+\lambda I)^{-1}\hat{\phi}_k\frac{d_k^{d}}{\Omega_{d}}\sum_{i=1}^{|D|}\sum_{i'\neq i}w_{i,s}w_{i',s}(f^*(x_i)-y_i)(f^*(x_{i'})-y_{i'})
   P_k^{d+1}(x_i\cdot
                  x_{i'})\\
                  &+&
   \sum_{k=0}^\infty(\hat{\phi}_k+\lambda I)^{-1}\hat{\phi}_k\frac{d_k^{d}}{\Omega_{d}}\sum_{i=1}^{|D|} w_{i,s}^2(f^*(x_i)-y_i)^2.
\end{eqnarray*}
Hence, Lemma \ref{Lemma:Hoeffding} together with $w_{i,s}\leq c_1|D|^{-1}$ implies that with confidence $1-\exp\left({-\frac{2t^2|D|^2}{M^{ 4}c_1^{ 4}}}\right)$, there holds
\begin{eqnarray*}
    \left\|(\mathcal L_\phi+\lambda I)^{-1/2}(L_{\phi,D,W_s}f^*-S_{D,W_s}^Ty_D)\right\|_\phi^2
    \leq
    (c_1^2{  M}^2|D|^{-1}+t)\sum_{k=0}^\infty(\hat{\phi}_k+\lambda I)^{-1}\hat{\phi}_k\frac{d_k^{d}}{\Omega_{d}}.
\end{eqnarray*}
Noting further $\hat{\phi}_k\sim k^{-2\gamma}$, $d_k^{d}\sim k^{d-1}$ and $\gamma>d/2$, we obtain
\begin{eqnarray*}
  &&\sum_{k=0}^\infty(\hat{\phi}_k+\lambda I)^{-1}\hat{\phi}_k\frac{d_k^{d}}{\Omega_{d}}
  \leq
  c_4\sum_{k=0}^\infty \frac{k^{-2\gamma+d-1}}{ k^{-2\gamma}+\lambda}
   \leq
  c_4\int_{1}^\infty \frac{t^{d-1}}{1+\lambda t^{2\gamma}}dt\\
  &=&
  \frac{c_4}d\int_{1}^\infty \frac{1}{1+\lambda t^{2\gamma}}dt^d
  \leq c_5 \lambda^{-\frac{d}{2\gamma}},
\end{eqnarray*}
where $c_4,c_5$ are constants depending only on $d$. Thus, with
 confidence $1-\exp\left({-\frac{2t^2|D|^2}{M^{ 4}c_1^{ 4}}}\right)$,
there holds
\begin{eqnarray*}
    \left\|(\mathcal L_\phi+\lambda I)^{-1/2}(L_{\phi,D,W_s}f^*-S_{D,W_s}^Ty_D)\right\|_\phi^2
    \leq
    c_6 \lambda^{-\frac{d}{2\gamma}}(t+|D|^{-1}),
\end{eqnarray*}
for $c_6$ a constant depending only on $d$. Setting $\delta=\exp\left({-\frac{2t^2|D|^2}{M^{ 4}c_1^{ 4}}}\right)$, we obtain
$t=\frac{M^2c_1^2}{\sqrt{2}|D|}$. Therefore,
  with confidence $1-\delta$, there holds
$$
    \left\|(\mathcal L_\phi+\lambda I)^{-1/2}(L_{\phi,D,W_s}f^*-S_{D,W_s}^Ty_D)\right\|_\phi
    \leq
    c_7 M \lambda^{-\frac{d}{4\gamma}} |D|^{-1/2}\log\frac3\delta,
$$
where $ {c}_7$ is a constant  depending only on $d$ and $c_1$.
This completes the proof of Proposition  \ref{Proposition:value-difference-random}.
\end{proof}

\section{Proofs}\label{Sec.proof}
In this section, we prove our main results by using the integral operator approach established in Section \ref{Sec:operator-approach}.

\subsection{Proof of Theorem \ref{Theorem:Generaliation-WRLS}}
For $f^*\in \mathcal N_\phi$,
define
\begin{equation}\label{noise-free-WRLS}
    f^\diamond_{D,W_s,\lambda}=\arg\min_{f\in\mathcal N_\phi}\sum_{(x_i,y_i)\in D}w_{i,s}(f(x_i)-f^*(x_i))^2+\lambda\|f\|_\phi^2
\end{equation}
as the noise-free version  of $f_{D,W_s,\lambda}$.
Then, it follows from Proposition \ref{Proposition:operator-KRR} that
\begin{equation}\label{noise-free-1}
    f^\diamond_{D,W_s,\lambda}=( L_{\phi,D,W_s}+\lambda I)^{-1}L_{\phi,D,W_s}f^*.
\end{equation}
Therefore,   we have
\begin{eqnarray}\label{Error-decomposition-global}
     \|f_{D,W_s,\lambda}-f^*\|_\psi
     \leq
     \overbrace{\|f^\diamond_{D,W_s,\lambda}-f^*\|_\psi}^{\mbox{approxmation error}}+\overbrace{\|f_{D,W_s,\lambda}- f^\diamond_{D,W_s,\lambda}\|_\psi}^{\mbox{Estimate error}}.
\end{eqnarray}
The following lemma provides  an estimate for the approximation error.

\begin{lemma}\label{Lemma:approximation-error-global}
Let $\mathcal Q_{\Lambda,s}:=\{(w_{i,s},  x_i): w_{i,s}\geq 0
\hbox{~and~}   x_i\in \Lambda\}$  be  a positive
 quadrature rule   on $\mathbb S^d$ with degree $s\in\mathbb N$.  If  $\hat{\phi}_k\sim k^{-2\gamma}$ with $\gamma>d/2$, $f^*\in \mathcal N_\phi$  and (\ref{kernel-relation}) holds with $0\leq r\leq 1$, then
\begin{equation}\label{App-1-error}
     \|f^\diamond_{D,W_s,\lambda}-f^*\|_\psi\leq (4\tilde{c}^2  \lambda^{-1} s^{-2\gamma}+2\tilde{c} s^{-\gamma}+\lambda)^{(1-r)/2}\|f^*\|_\phi.
\end{equation}
\end{lemma}

\begin{proof}
Since $f^*\in \mathcal N_\phi$ and (\ref{kernel-relation}) holds, we have from (\ref{norm-operator-2}), (\ref{operator-equivalent}) and (\ref{noise-free-1}) that
\begin{eqnarray*}
  &&\|f^*-f^\diamond_{D,W_s,\lambda}\|_\psi
  =\|L_\phi^{(1-r)/2}(f^*-f^\diamond_{D,W_s,\lambda})\|_\phi\\
  &=&
   \|L_\phi^{(1-r)/2}[(L_{\phi,D,W_s}+\lambda I)^{-1}L_{\phi,D,W_s}-I]f^*\|_\phi\\
   &=&
    \lambda \|\mathcal L_\phi^{(1-r)/2}(L_{\phi,D,W_s}+\lambda I)^{-1}f^*\|_\phi
    \leq
    \lambda \|(\mathcal L_\phi+\lambda I)^{(1-r)/2}(L_{\phi,D,W_s}+\lambda I)^{-1}f^*\|_\phi.
\end{eqnarray*}
Then, we obtain from Proposition \ref{Proposition:Product} with $v=0$ and the well known Cordes inequality \cite{Bathia1997}
\begin{equation}\label{Cordes}
   \|A^uB^u\|\leq \|AB\|^u,\qquad 0\leq u\leq1
\end{equation}
for positive operators $A,B$
  that
\begin{eqnarray*}
  &&\|f^*-f^\diamond_{D,W_s,\lambda}\|_\psi
   \leq
  \lambda\|(\mathcal L_\phi+\lambda I)^{(1-r)/2}(\mathcal L_{\phi,D,W_s}+\lambda I)^{(r-1)/2}\|
  \|(L_{\phi,D,W_s}+\lambda I)^{-(r+1)/2}f^*\|_\phi\\
  &\leq&
  \lambda(4\tilde{c}^2  \lambda^{-2} s^{-2\gamma}+2\tilde{c}\lambda^{-1}s^{-\gamma}+1)^{(1-r)/2}\lambda^{-(r+1)/2}\|f^*\|_\phi\\
  &=&
 (4\tilde{c}^2  \lambda^{-1} s^{-2\gamma}+2\tilde{c} s^{-\gamma}+\lambda)^{(1-r)/2}\|f^*\|_\phi.
\end{eqnarray*}
This completes the proof of Lemma \ref{Lemma:approximation-error-global}.
\end{proof}

Next, we aim to bound the estimate error, as shown in the following lemma.

\begin{lemma}\label{Lemma:sample-error-global}
Let $0<\delta<1$ and $\mathcal Q_{\Lambda,s}:=\{(w_{i,s},  x_i): w_{i,s}\geq 0
\hbox{~and~}   x_i\in \Lambda\}$  be  a positive
 quadrature rule   on $\mathbb S^d$ with degree $s\in\mathbb N$ satisfying $0<w_{i,s}\leq c_1|D|^{-1}$.  If  $\hat{\phi}_k\sim k^{-2\gamma}$ with $\gamma>d/2$, $f^*\in \mathcal N_\phi$, (\ref{kernel-relation}) holds with $0\leq r\leq 1$ and $y_i=f(x_i)+\varepsilon_i$ with $\varepsilon_i$ i.i.d. random noise satisfying $E[\varepsilon_i=0]$ and $|\varepsilon_i|\leq M$ for some $M>0$,
  then with confidence $1-\delta$, there holds
\begin{equation}\label{Sample-1-error}
     \|f_{D,W_s,\lambda}- f^\diamond_{D,W_s,\lambda}\|_\psi
     =
     \tilde{c}'M \lambda^{-\frac{2r\gamma+d}{4\gamma}} (4\tilde{c}^2  \lambda^{-2} s^{-2\gamma}+2\tilde{c}\lambda^{-1}s^{-\gamma}+1)
      |D|^{-1/2}\log\frac3\delta.
\end{equation}
\end{lemma}

\begin{proof}
Since $f^*\in \mathcal N_\phi$ and (\ref{kernel-relation}) holds, we have from (\ref{operator-KRR}) and  (\ref{noise-free-1}) that
\begin{eqnarray*}
  &&\|f_{D,W_s,\lambda}- f^\diamond_{D,W_s,\lambda}\|_\psi
  =\|L_\phi^{(1-r)/2}(f_{D,W_s,\lambda}- f^\diamond_{D,W_s,\lambda})\|_\phi\\
  &=&
   \|L_\phi^{(1-r)/2}(L_{\phi,D,W_s}+\lambda I)^{-1}(S^T_{D,W_s}y_D-\mathcal L_\phi f^*)\|_\phi\\
   &\leq& \lambda^{-r/2}
    \|(\mathcal L_\phi+\lambda I)^{1/2}(L_{\phi,D,W_s}+\lambda I)^{-1}(S^T_{D,W_s}y_D-\mathcal L_\phi f^*)\|_\phi\\
    &\leq&
    \lambda^{-r/2}\|(\mathcal L_\phi+\lambda I)^{1/2}(L_{\phi,D,W_s}+\lambda I)^{-1/2}\|^2
    \|(\mathcal L_\phi+\lambda I)^{-1/2}(S^T_{D,W_s}y_D-\mathcal L_\phi f^*)\|_\phi.
\end{eqnarray*}
Therefore, Proposition  \ref{Proposition:Product} with $v=0$, Proposition \ref{Proposition:value-difference-random} and (\ref{Cordes}) yield that with confidence $1-\delta$, there holds
\begin{eqnarray*}
     &&\|f_{D,W_s,\lambda}- f^\diamond_{D,W_s,\lambda}\|_\psi
     =
     \lambda^{-r/2} (4\tilde{c}^2  \lambda^{-2} s^{-2\gamma}+2\tilde{c}\lambda^{-1}s^{-\gamma}+1)
     \tilde{c}' M \lambda^{-\frac{d}{4\gamma}} |D|^{-1/2}\log\frac3\delta.
\end{eqnarray*}
This completes the proof of Lemma \ref{Lemma:sample-error-global}.
\end{proof}

Now we are in a position to prove Theorem \ref{Theorem:Generaliation-WRLS}.

\begin{proof}[Proof of Theorem \ref{Theorem:Generaliation-WRLS}]
Plugging (\ref{Sample-1-error}) and (\ref{App-1-error}) into (\ref{Error-decomposition-global}), we obtain that with confidence $1-\delta$, there holds
\begin{eqnarray}\label{global-error-1}
      \|f_{D,W_s,\lambda}-f^*\|_\psi
     &\leq&
     (4\tilde{c}^2  \lambda^{-1} s^{-2\gamma}+2\tilde{c} s^{-\gamma}+\lambda)^{(1-r)/2}\|f^*\|_\phi \nonumber\\
     &+&
     \tilde{c}'M \lambda^{-\frac{2r\gamma+d}{4r}} (4\tilde{c}^2  \lambda^{-2} s^{-2\gamma}+2\tilde{c}\lambda^{-1}s^{-\gamma}+1)
      |D|^{-1/2}\log\frac3\delta.
\end{eqnarray}
Then for   $s\geq \lambda^{-1/\gamma}$, the above estimate yields
$$
   \|f_{D,W_s,\lambda}-f^*\|_\psi
   \leq
  c_8\left(\lambda^{\frac{1-r}2}+\lambda^{-\frac{2r\gamma+d}{4\gamma}}|D|^{-1/2}\right)\log\frac3\delta,
$$
where $c_8:=(1+4\tilde{c}^2+2\tilde{c})\max\{\|f^*\|_\phi,\tilde{c}'M\}.$ Noting further $\lambda\sim|D|^{-\frac{2\gamma}{2\gamma +d}}$, we get that with confidence $1-\delta$, there holds
$$
   \|f_{D,W_s,\lambda}-f^*\|_\psi\leq C|D|^{-\frac{(1-r)\gamma}{2\gamma+d}}\log\frac3\delta,
$$
where $C$ is a constant depending only on $d$, $M$ and $\|f^*\|_\phi$. The proof of Theorem \ref{Theorem:Generaliation-WRLS} is completed.
\end{proof}

\subsection{Proof of Theorem \ref{Theorem:Generaliation-DWRLS}}

To prove Theorem \ref{Theorem:Generaliation-DWRLS}, we need the following error decomposition strategy.

\begin{lemma}\label{Lemma:Error-decomposition}
Let $\overline{f}_{D,W_{\vec{s}},\vec{\lambda}}$ be defined by (\ref{DWRLS}). Then, we have
\begin{eqnarray}\label{error-dec-distributed}
        &&  E[\|\overline{f}_{D,W_{\vec{s}},\vec{\lambda}}-f^*\|_{\psi}^2] \nonumber\\
        & \leq&
         \sum_{j=1}^m\frac{|D_j|^2}{|D|^2}  E\left[\|f_{D_j,W_{j,s_j},\lambda_j}
         -f^*\|_{\psi}^2\right]
        +\sum_{j=1}^m\frac{|D_j|}{|D|}
        \left\|f^\diamond_{D_j,W_{j,s_j},\lambda_j}-f^*\right\|_{\psi}^2,
\end{eqnarray}
where $f^{\diamond}_{D_j,W_{j,s_j},\lambda_j}$ is the noise-free version of  $f_{D_j,W_{j,s_j},\lambda_j}$
defined by
\begin{equation}\label{noise-free-2}
    f^\diamond_{D_j,W_{j,s_j},\lambda_j}=( L_{\phi,D_j,W_{j,s_j}}+\lambda_j I)^{-1}L_{\phi,D_j,W_{j,s}}f^*.
\end{equation}
\end{lemma}

\begin{proof} Due to
$\sum_{j=1}^m\frac{|D_j|}{|D|}=1$, we have
\begin{align*}
       &\|\overline{f}_{D,W_{\vec{s}},\vec{\lambda}}-f^*\|_{\psi}^2
        =
       \left\|\sum_{j=1}^m\frac{|D_j|}{|D|}(f_{D_j,W_{j,s_j},\lambda_j}-f^*)\right\|_{\psi}^2\\
       &\hspace{-2mm}=
       \sum_{j=1}^m\frac{|D_j|^2}{|D|^2}\|f_{D_j,W_{j,s_j},\lambda_j}-f^*\|_{\psi}^2
       +
       \sum_{j=1}^m\frac{|D_j|}{|D|}\left\langle
       f_{D_j,W_{j,s_j},\lambda_j}-f^*,\sum_{k\neq
       j}\frac{|D_k|}{|D|}(f_{D_k,W_{k,s_k},\lambda_k}-f^*)\right\rangle_{\psi}.
\end{align*}
Taking expectations, we have
\begin{align*}
        &  E\left[\|\overline{f}_{D,W_{\vec{s}},\vec{\lambda}}-f^*\|_{\psi}^2\right]
        =
        \sum_{j=1}^m\frac{|D_j|^2}{|D|^2}  E\left[ \left\|f_{D_j,W_{j,s_j},\lambda_j}-f^*\right\|_{\psi}^2\right]\\
        &\quad+
        \sum_{j=1}^m\frac{|D_j|}{|D|}\left\langle
         E_{D_j}[f_{D_j,W_{j,s_j},\lambda_j}]-f^*,  E[\overline{f}_{D,W_{\vec{s}},\vec{\lambda}}]-f^*- \frac{|D_j|}{|D|}
       \left( E_{D_j}[f_{D_j,W_{j,s_j},\lambda_j}]-f^*\right)\right\rangle_\phi.
\end{align*}
But
$$
        \sum_{j=1}^m\frac{|D_j|}{|D|}\left\langle
         E_{D_j}[f_{D_j,W_{j,s_j},\lambda_j}]-f^*, E[\overline{f}_{D,W_{\vec{s}},\vec{\lambda}}]-f^*\right\rangle_{\psi}\\
      	 =
       \|E[\overline{f}_{D,W_{\vec{s}},\vec{\lambda}}]-f^*\|_{\psi}^2.
$$
Then,
\begin{align*}
        &E\left[\|\overline{f}_{D,W_{\vec{s}},\vec{\lambda}}-f^*\|_{\psi}^2\right]
          =
         \sum_{j=1}^m\frac{|D_j|^2}{|D|^2}  E\left[ \left\|f_{D_j,W_{j,s_j},\lambda_j}-f^*\right\|_{\psi}^2\right]\\
         &\qquad-\sum_{j=1}^m\frac{|D_j|^2}{|D|^2}\left\| E[f_{D_j,W_{j,s_j},\lambda_j}]-f^*\right\|_{\psi}^2 +
        \|E[\overline{f}_{D,W_{\vec{s}},\vec{\lambda}}]-f^*\|_{\psi}^2.
\end{align*}
Noting further from (\ref{noise-free-2}) that
$$
      E[\overline{f}_{D,W_{\vec{s}},\vec{\lambda}}]=\sum_{j=1}^m\frac{|D_j|}{|D|}f^\diamond_{D_j,W_{j,s_j},\lambda_j},
$$
we obtain from
$\sum_{j=1}^m\frac{|D_j|}{|D|}=1$ that
$$
    \|E[\overline{f}_{D,W_{\vec{s}},\vec{\lambda}}]-f^*\|_{\psi}^2
     =
       \left\|\sum_{j=1}^m\frac{|D_j|}{|D|}\left(f^\diamond_{D_j,W_{j,s_j},\lambda_j}-f^*\right)\right\|_{\psi}^2\\[1mm]
     \leq
       \sum_{j=1}^m\frac{|D_j|}{|D|}\left\|f^\diamond_{D_j,W_{j,s_j},\lambda_j}-f^*\right\|_{\psi}^2,
$$
which  proves the bound in \eqref{error-dec-distributed}. The proof of Lemma \ref{Lemma:Error-decomposition} is completed.
\end{proof}

For further analysis, the following lemma is   needed.

\begin{lemma}\label{Lemma:prob-to-exp}
Let   $\xi$ be a random variable with nonnegative values. If $\xi\leq \mathcal A\log^b\frac{c}{\delta}$ holds with confidence $1-\delta$  for some $\mathcal A,b,c>0$ and any $0<\delta<1$, then
$$
      E[\xi]\leq c\Gamma(b+1) \mathcal A,
$$
where $\Gamma(\cdot)$ is the Gamma function.
\end{lemma}

\begin{proof}
Since $\xi\leq \mathcal A\log^b\frac{c}{\delta}$ holds with confidence $1-\delta$, we have for any $t>0$,
$$
    P[\xi>t]\leq c\exp\{-\mathcal A^{-1/b}t^{1/b}\}.
$$
Using the probability to expectation formula
$$
               E[\xi] =\int_0^\infty P\left[\xi > t\right] d t
$$
  to the   random variable $\xi$ with nonnegative values, we obtain
\begin{eqnarray*}
     E[\xi]\leq  c\int_{0}^\infty\exp\{-\mathcal A^{-1/b}t^{1/b}\}dt
     \leq c\mathcal A\Gamma(b+1).
\end{eqnarray*}
This completes the proof of Lemma \ref{Lemma:prob-to-exp}.
\end{proof}

Based on the above lemmas, we can prove Theorem \ref{Theorem:Generaliation-DWRLS} as follows.

\begin{proof}[Proof of Theorem \ref{Theorem:Generaliation-DWRLS}]
For any $1\leq j\leq m$,  since $\hat{\phi}_k\sim k^{-2\gamma}$ with $\gamma>d/2$, $f^*\in \mathcal N_\phi$, (\ref{kernel-relation}) holds,
$\Lambda_j=\{x_{i,j}\}_{i=1}^{|D_j|}$ be  $\tau_j$-quasi uniform and  $\mathcal Q_{\Lambda_j,s_j}:=\{(w_{i,j,s_j},  x_{i,j}): w_{i,j,s}\geq 0
\hbox{~and~}   x_{i,j}\in \Lambda_j\}$  be a quadrature rule satisfying $0<w_{i,j,s_j}\leq c_1|D_j|^{-1}$, it follows from Lemma \ref{Lemma:approximation-error-global} with $\Lambda$ being replaced by $\Lambda_j$ that
\begin{equation}\label{App-2-error}
     \|f^\diamond_{D_j,W_{j,s_j},\lambda_j}-f^*\|_\psi\leq (4\tilde{c}^2  \lambda_j^{-1} s_j^{-2\gamma}+2\tilde{c} s_j^{-\gamma}+\lambda_j)^{(1-r)/2}\|f^*\|_\phi.
\end{equation}
Furthermore, we obtain from (\ref{global-error-1}) with $D$ being replaced by $D_j$ that with confidence $1-\delta$, there holds
\begin{eqnarray}\label{global-error-2}
      \|f_{D_j,W_{j,s_j},\lambda_j}-f^*\|_\psi
     &\leq&
     (4\tilde{c}^2  \lambda_j^{-1} s_j^{-2\gamma}+2\tilde{c} s_j^{-\gamma}+\lambda_j)^{(1-r)/2}\|f^*\|_\phi \nonumber\\
     &+&
     \tilde{c}'M \lambda_j^{-\frac{2r\gamma+d}{4\gamma}} (4\tilde{c}^2  \lambda_j^{-2} s_j^{-2\gamma}+2\tilde{c}\lambda_j^{-1}s_j^{-\gamma}+1)
      |D_j|^{-1/2}\log\frac3\delta.
\end{eqnarray}
Applying Lemma \ref{Lemma:prob-to-exp} with $\mathcal A=2(\tilde{c}')^2 M^2\lambda_j^{-\frac{2r\gamma+d}{2\gamma}} (4\tilde{c}^2  \lambda_j^{-2} s_j^{-2\gamma}+2\tilde{c}\lambda_j^{-1}s_j^{-\gamma}+1)^2
      |D_j|^{-1}$, $b=2$ and $c=3$, we get from (\ref{global-error-2}) that
\begin{eqnarray}\label{global-error-3}
   &&E[\|f_{D_j,W_{j,s_j},\lambda_j}-f^*\|^2_\psi]\leq 2(4\tilde{c}^2  \lambda_j^{-1} s_j^{-2\gamma}+2\tilde{c} s_j^{-\gamma}+\lambda_j)^{1-r}\|f^*\|^2_\phi \nonumber\\
   &+&
   12(\tilde{c}')^2M^2 \lambda_j^{-\frac{2r\gamma+d}{2\gamma}} (4\tilde{c}^2  \lambda_j^{-2} s_j^{-2\gamma}+2\tilde{c}\lambda_j^{-1}s_j^{-\gamma}+1)^2
      |D_j|^{-1}.
\end{eqnarray}
Plugging (\ref{global-error-3}) and (\ref{App-2-error}) into (\ref{error-dec-distributed}), we have
\begin{eqnarray*}
    &&  E[\|\overline{f}_{D,W_{\vec{s}},\vec{\lambda}}-f^*\|_{\psi}^2] \leq
    2\sum_{j=1}^m\left(\frac{|D_j|^2}{|D|^2}+ \frac{|D_j|}{|D|}\right) (4\tilde{c}^2  \lambda_j^{-1} s_j^{-2\gamma}+2\tilde{c} s_j^{-\gamma}+\lambda_j)^{1-r}\|f^*\|^2_\phi\\
    &+&
    12(\tilde{c}')^2M^2\sum_{j=1}^m\frac{|D_j|^2}{|D|^2}\lambda_j^{-\frac{2r\gamma+d}{2\gamma}} (4\tilde{c}^2  \lambda_j^{-2} s_j^{-2\gamma}+2\tilde{c}\lambda_j^{-1}s_j^{-\gamma}+1)^2
      |D_j|^{-1}.
\end{eqnarray*}
Since $s_j\geq\lambda_j^{-1/\gamma}$,  we have
$$
    E[\|\overline{f}_{D,W_{\vec{s}},\vec{\lambda}}-f^*\|_{\psi}^2]
    \leq
      {c}_9  \sum_{j=1}^m\frac{|D_j|}{|D|}\left(\lambda_j^{1-r}+\frac1{|D|} \lambda_j^{-\frac{2r\gamma+d}{2\gamma}}\right),
$$
where $ {c}_9$ is a constant depending only on $d$, $r$, $\gamma$, $M$ and $\|f^*\|_\phi$.
Noting further $\lambda_j\sim |D|^{-\frac{2\gamma}{2\gamma+d}}$, the above estimate yields
$$
   E[\|\overline{f}_{D,W_{\vec{s}},\vec{\lambda}}-f^*\|_{\psi}^2]
   \leq
   2\bar{c}_8 |D|^{-\frac{2(1-r)\gamma}{2\gamma+d}}.
$$
This completes the proof of Theorem \ref{Theorem:Generaliation-DWRLS}.
\end{proof}

\section{Simulations}\label{sec.Simulation}
In this section, we conduct several  numerical simulations to verify our theoretical statements and show the excellent performance of DWRLS. We compare the following three methods: distributed filtered hyperinterpolation (DFH) proposed in \cite{Lin2021}, WRLS, and DWRLS, where WRLS with training all samples in a batch mode is considered as a baseline. Two functions are used to generate samples for simulation. The first function is constructed via the  Wendland function  \cite{Chernih2014}
\begin{equation}
\tilde{\Psi}(u) = (1-u)_{+}^8(32u^3+25u^2+8u+1),
\end{equation}
where $u_+ = \max\{u, 0\}$, and it is defined by
\begin{equation}
f_1({x}) = \sum\limits_{i=1}^{10} \tilde{\Psi}(\|{x}-{z}_i\|_2),
\end{equation}
where ${z}_i$ ($i=1,\cdots,10$) are the center points of the regions of an equal area partitioned by Leopardi's recursive zonal sphere partitioning
procedure  \cite{Leopardi2006}.
The second function is the Franke function, modified by Renka \cite{Renka1988},
\begin{eqnarray}
f_2(x) &=& 0.75\exp(-(9x^{(1)}-2)^2/4-(9x^{(2)}-2)^2/4-(9x^{(3)}-2)^2/4) \nonumber\\
            &+&   0.75\exp(-(9x^{(1)}+1)^2/49-(9x^{(2)}+1)/10-(9x^{(3)}+1)/10) \nonumber\\
            &+&   0.5\exp(-(9x^{(1)}-7)^2/4-(9x^{(2)}-3)^2/4-(9x^{(3)}-5)^2/4) \nonumber\\
            &-&   0.2\exp(-(9x^{(1)}-4)^2-(9x^{(2)}-7)^2-(9x^{(3)}-5)^2),
\end{eqnarray}
where ${x}=(x^{(1)}, x^{(2)}, x^{(3)})^T$.

Before describing the simulations, we introduce the generation process of the simulation data as follows: First, Womersley's symmetric spherical $45$-designs \cite{Womersley2018} are used to generate
$1038$ points $\{\hat{x}_{i}\}_{i=1}^{1038}$ on the unit sphere. Second, the points $\{\hat{x}_{i}\}_{i=1}^{1038}$ are rotated by the rotation matrix
$$
A_j := \left(
  \begin{array}{ccc}
    \cos(j\pi/10) & -\sin(j\pi/10) & 0 \\
    \sin(j\pi/10) & \cos(j\pi/10) & 0 \\
    0 & 0 & 1 \\
  \end{array}
\right)
$$
to obtain   new points $\{{x}_{i,j}\}_{i=1}^{1038}$ for $j=1,\cdots,10$, i.e., ${x}_{i,j}=A_j\hat{x}_{i}$, and the points $\{x_{i,j}\}_{i=1,j=1}^{1038,10}$ are used as the inputs of
training samples; the corresponding outputs $\{\{y_{i,j}\}_{i=1}^{1038}\}_{j=1}^{10}$ are generated by
\begin{equation}
y_{i,j} = f_j(x_{i,j}) + \varepsilon_{i,j}, \quad i = 1, \cdots, 1038, \quad j = 1, \cdots, 10,
\end{equation}
where $\varepsilon_{i,j}$ is the independent Gaussian noise $\mathcal{N}(0, 0.1^2)$. In this way,  there are a total of $N=10380$ training samples.
Finally, the inputs $\{x_i'\}_{i=1}^{N'}$  of testing samples are $N'=10000$ generalized
spiral points on the unit sphere, and the corresponding outputs $\{y_i'\}_{i=1}^{N'}$ are generated by $y_i'=f_j(x_i')$.

Some implementation details of simulations are described as follows. For DFH, the training samples are equally distributed to $10$ local machines,
i.e., the samples $D_j:=\{(x_{i,j}, y_{i,j})\}_{i=1}^{1038}$ obtained by the $j$th rotation matrix are located on the $k$th local machine, and the
parameter $L$ related to the polynomial degree is selected from the set $\{2,4,\cdots,40\}$.
For DWRLS, the training samples are distributed to $m$ ($m\geq 10$) local machines in the following way: Let $\tau:=\mod(m,10)$. If $\tau=0$, i.e.,
$m$ can be divided by $10$, then the samples of each set $D_j$ are randomly and equally distributed to $m/10$ local machines. If $r>0$,
we randomly choose $r$ sets from $\{D_j\}_{j=1}^{10}$; the samples of each chosen set are equally distributed to $\lceil m/10 \rceil$
local machines; the samples of each set of the remaining $10-\tau$ sets are equally distributed to $\lfloor m/10 \rfloor$
local machines. In the execution of WRLS and DWRLS, we use the positive
definite function $\phi_2(x, x')=\tilde{\Psi}(\|x - x'\|_2)$ with the regularization parameter $\lambda$ being chosen from the set
$\{\frac{1}{2^q} | \frac{1}{2^q}>10^{-10}, q=0,1,2,\cdots\}$ for the approximation of function $f_1$.
For the approximation of function $f_2$, the positive definite function is defined as
$\phi_1(x, x')=\exp(-\frac{\|x - x'\|_2^2}{2\sigma^2})$ with the regularization parameter $\lambda$ being chosen from the set
$\{\frac{1}{3^q} | \frac{1}{3^q}>10^{-10}, q=0,1,2,\cdots\}$  and the width $\sigma$ being chosen from 10 values which are drawn in a logarithmic, equally spaced interval
$[0.1, 1]$. All parameters in the simulations are selected by grid search.

Our first aim of simulation is to  compare DWRLS with DFH. The result is shown in Figure \ref{RMSEcontrast}. It can be found in Figure \ref{RMSEcontrast} that for the same partitions of data, DWRLS is at least as good as DFH. This is not a surprising phenomenon since our theoretical assertions are made under the Sobolev error estimate while the result of DFH \cite{Lin2021} is only  derived in $L^2(\mathbb S^d)$, showing that our analysis is available to numerous types of data.

\begin{figure*}[t]
    \centering
    \includegraphics[width=6cm,height=4cm]{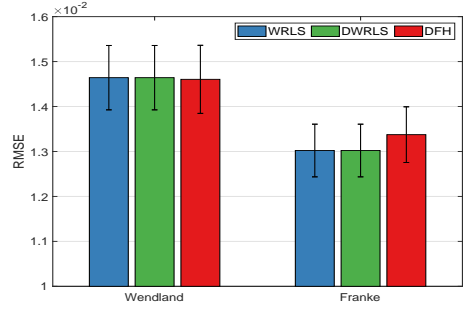}
    \caption{Comparison of testing RMSE when the number of local servers is set to 10}\label{RMSEcontrast}
\end{figure*}

Our second aim is to show the relation between the fitting performance of DWRLS and the number of local servers. The numerical results is shown in Figure \ref{RMSEvsm}. There are two interesting findings from  Figure \ref{RMSEvsm}: 1) The RMSE of DWRLS is somewhat stable with the number of severs $m$. Taking the Franke function for example, the RMSE changes from 0.013 to 0.020 when the number of servers increases from 1 to 100; 2) When $m$ is not so large, DWRLS performs similarly as WRLS, which verifies our theoretical assertions. Both findings show the excellent performance of DWRLS in fitting noisy spherical data.

\begin{figure*}[t]
    \centering
    \subfigcapskip=-2pt
    \subfigure[Wendland]{\includegraphics[width=5cm,height=4cm]{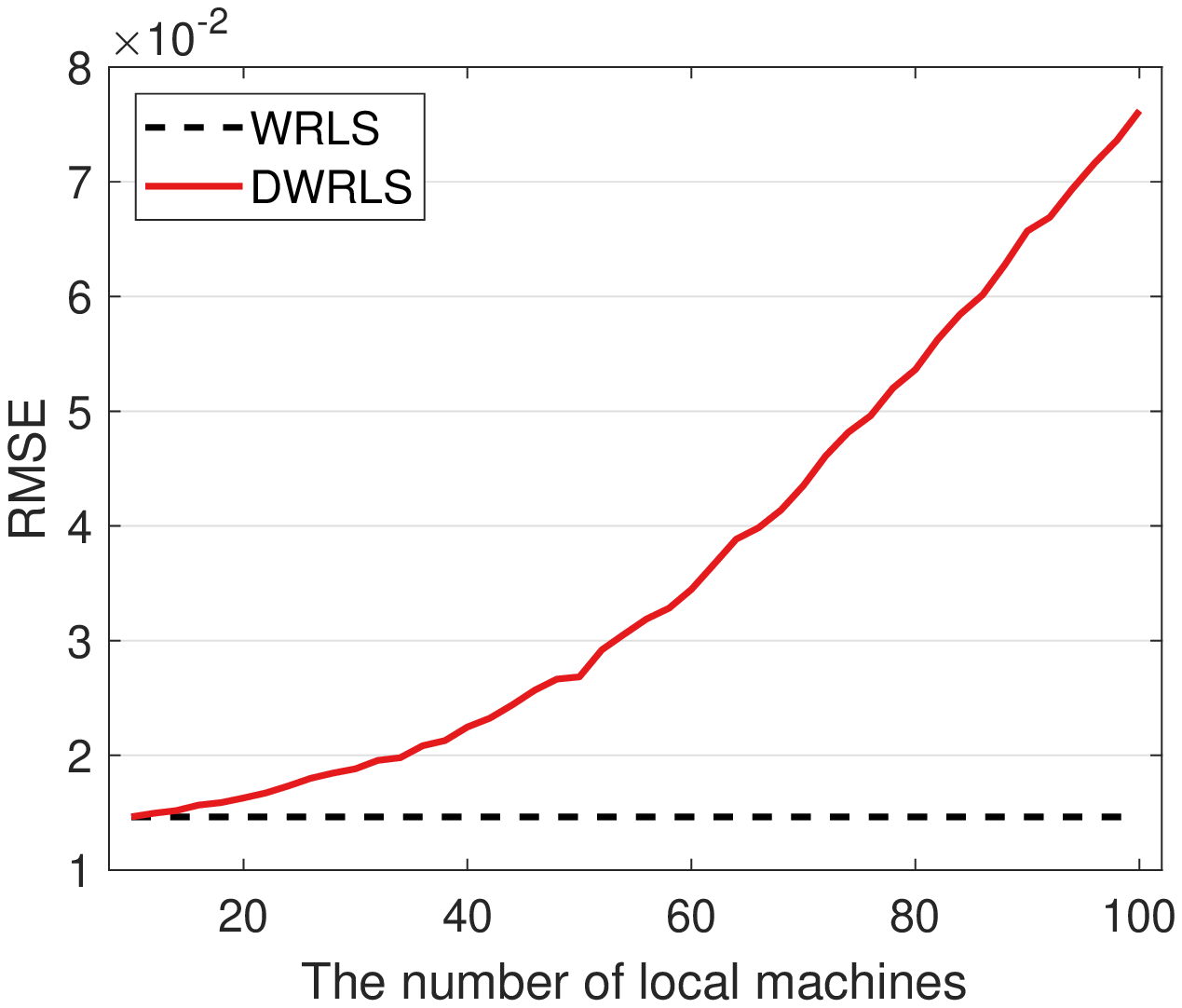}}\hspace{0.5in}
    \subfigure[Franke]{\includegraphics[width=5cm,height=4cm]{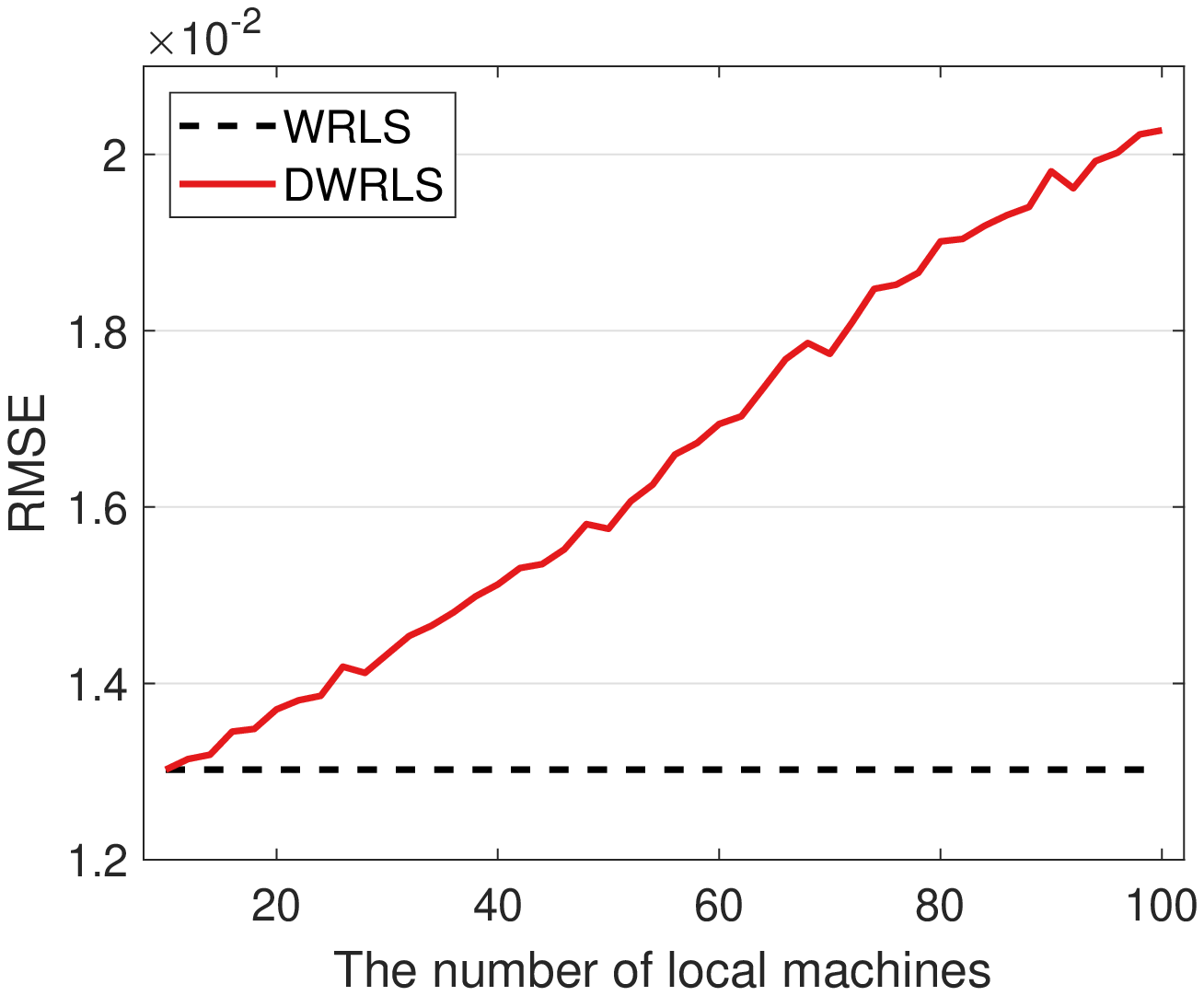}}
    \caption{Relation between the RMSE and the number of local servers for DWRLS}\label{RMSEvsm}
\end{figure*}

Our third purpose is to explicitly demonstrate the performance of DWLRS via plotting the ground truth functions, the fitted functions and the approximation error. The first rows of Figures \ref{Comparisonf1} and \ref{Comparisonf2} concern the ground truth functions and fitted functions, while the second rows exhibit the noise of ground truth and approximation errors. In this simulation, the number of local servers of DFH is always set to be 10 while that of DWRLS is set to be 10, 50, 100, respectively. From these two figures,  it can be found that the fitted functions are very similar as the ground truth functions for DWRLS with not so large $m$, i.e., $m=10,50$. Furthermore, the fitted RMSE is near to zero for DWRLS with $m=10,50$, though it is different from the ground truth noise. The main reason is that the regularization term $\lambda\|f\|_\phi^2$ in the definition of DWRLS provides a trade-off between the approximation error and estimate error in \eqref{Error-decomposition-global}. Under this circumstance, the fitted error can be much smaller than the ground truth error, provided the regularization parameter $\lambda$ is appropriately tuned. This also verifies the power of DWRLS in fitting noisy spherical data.

\begin{figure*}[t]
    \centering
    \subfigcapskip=-1pt
    \subfigure{\includegraphics[width=2.15cm,height=2.5cm]{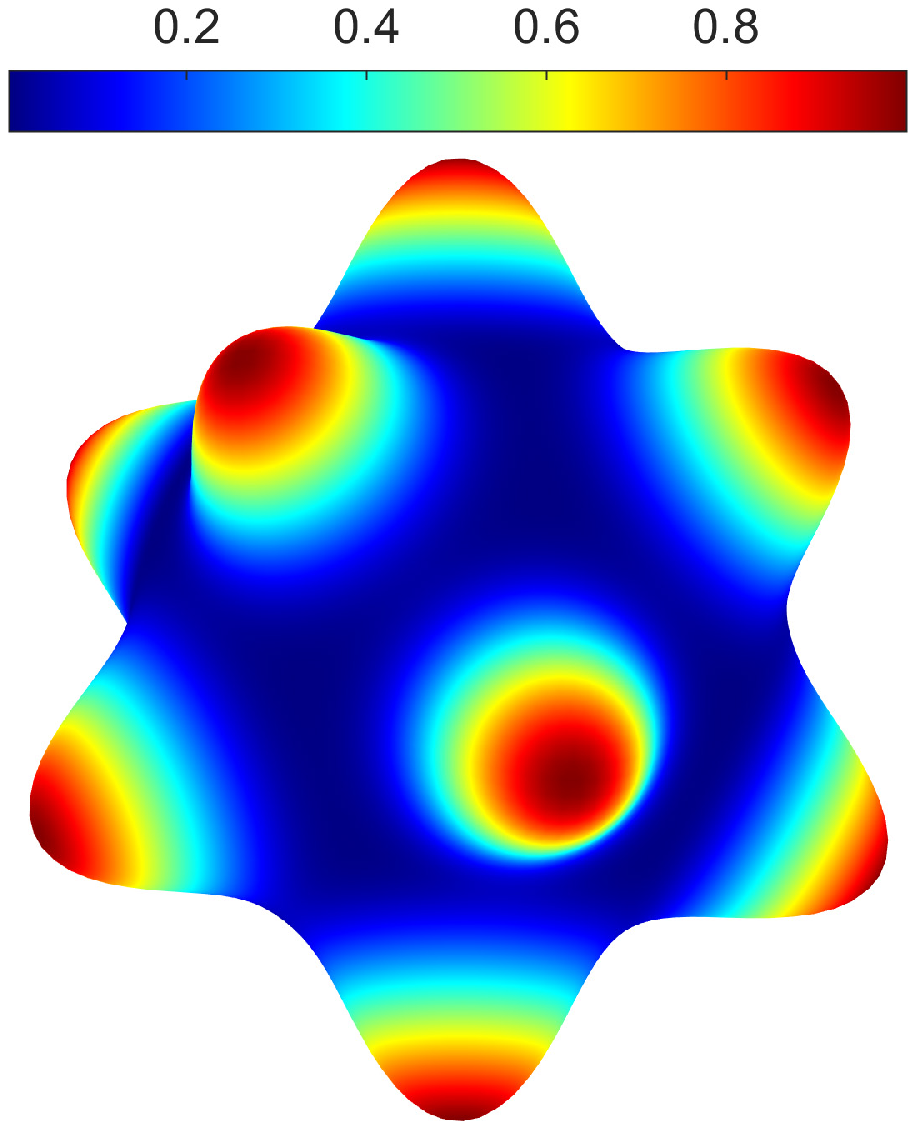}}
    \subfigure{\includegraphics[width=2.15cm,height=2.5cm]{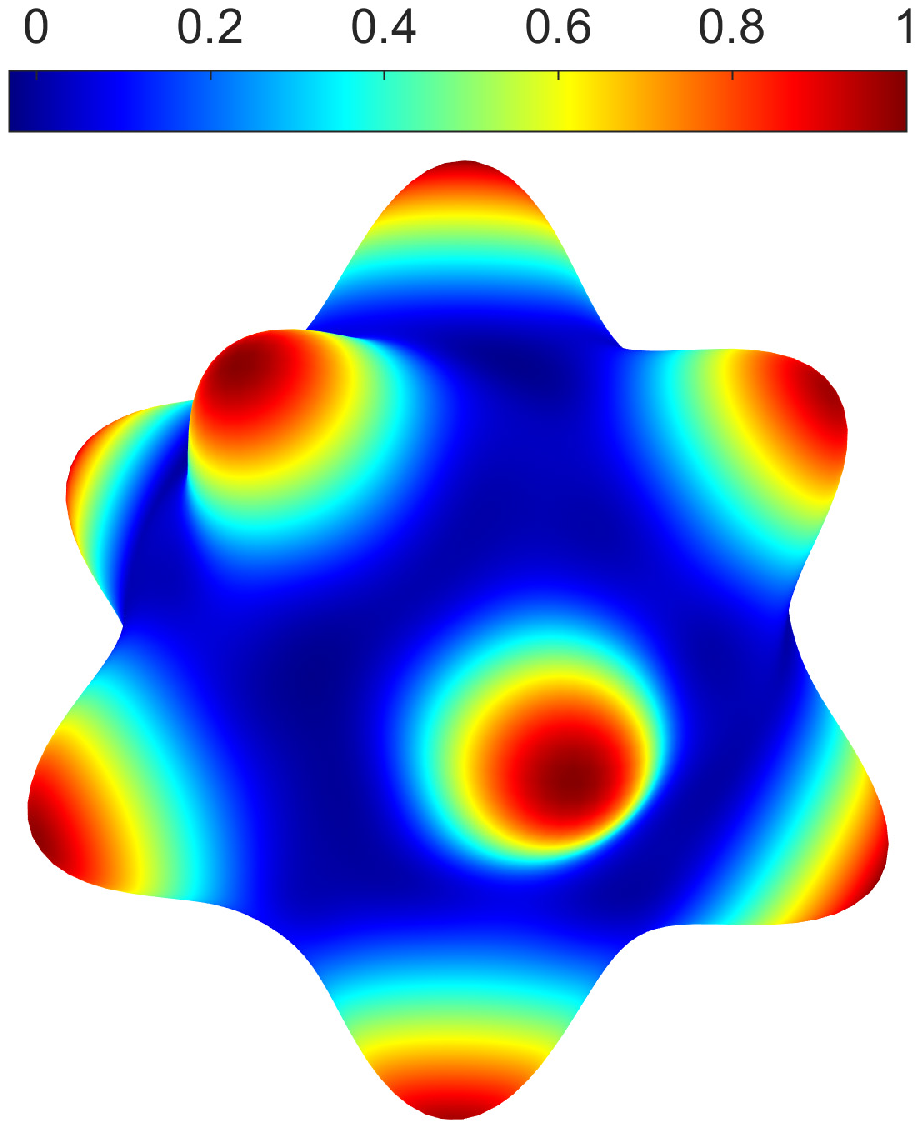}}
    \subfigure{\includegraphics[width=2.15cm,height=2.5cm]{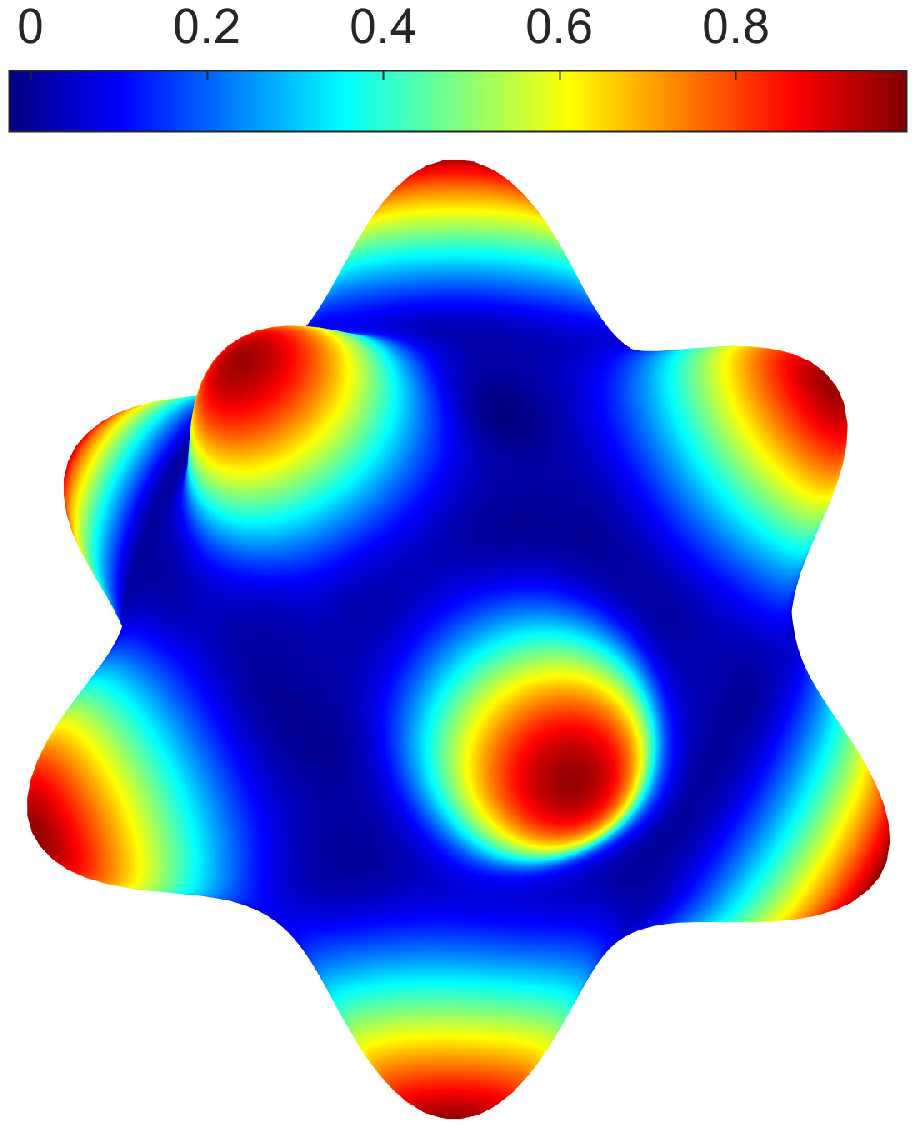}}
    \subfigure{\includegraphics[width=2.15cm,height=2.5cm]{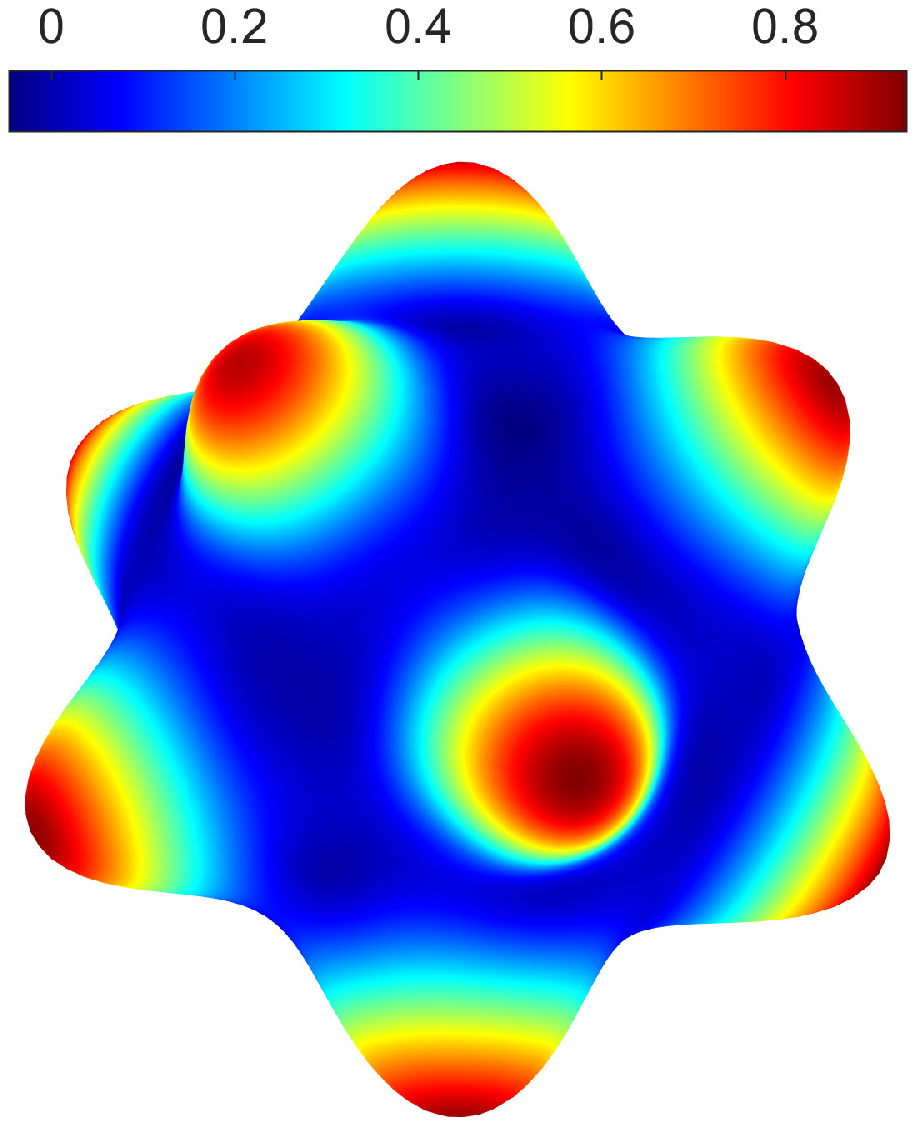}}
    \subfigure{\includegraphics[width=2.15cm,height=2.5cm]{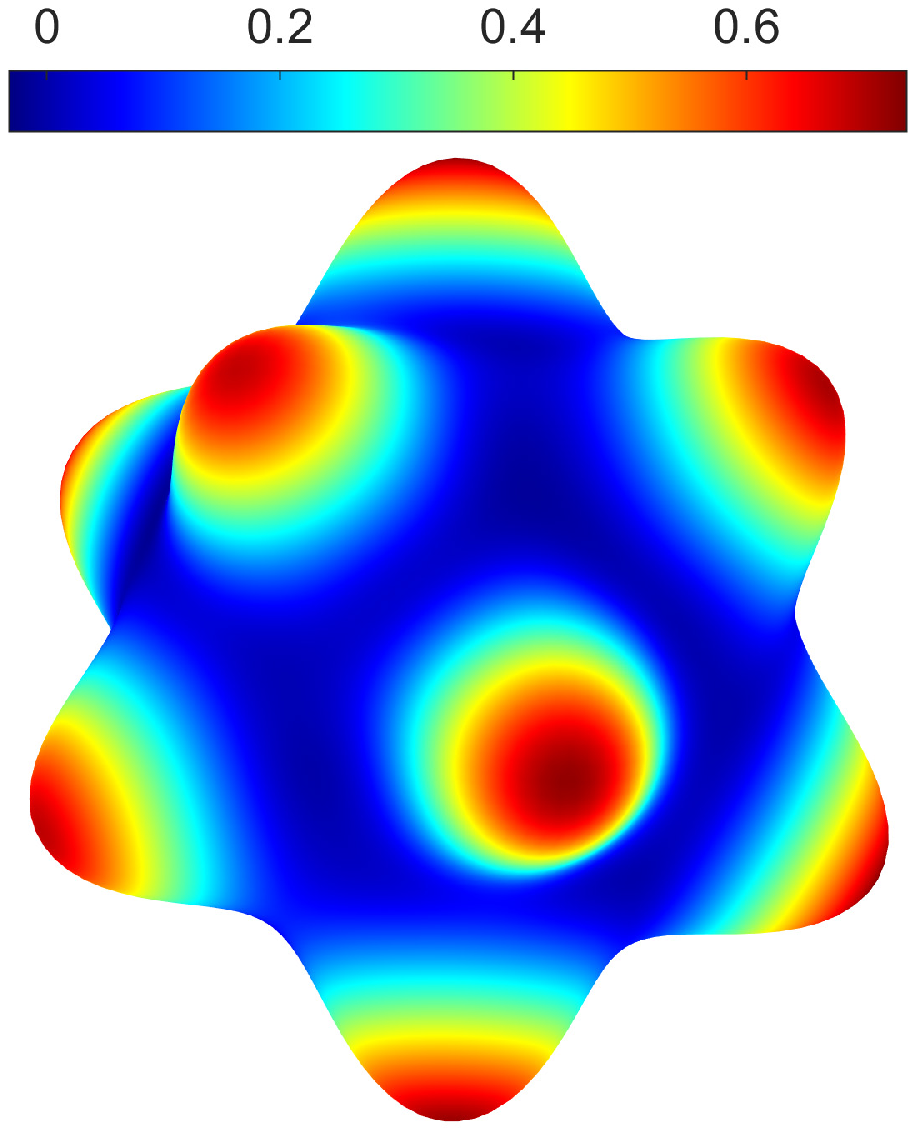}}\\
    \vspace{-0.08in}
    \setcounter{subfigure}{0}
    \subfigure[Groundtruth]{\includegraphics[width=2.15cm,height=2.5cm]{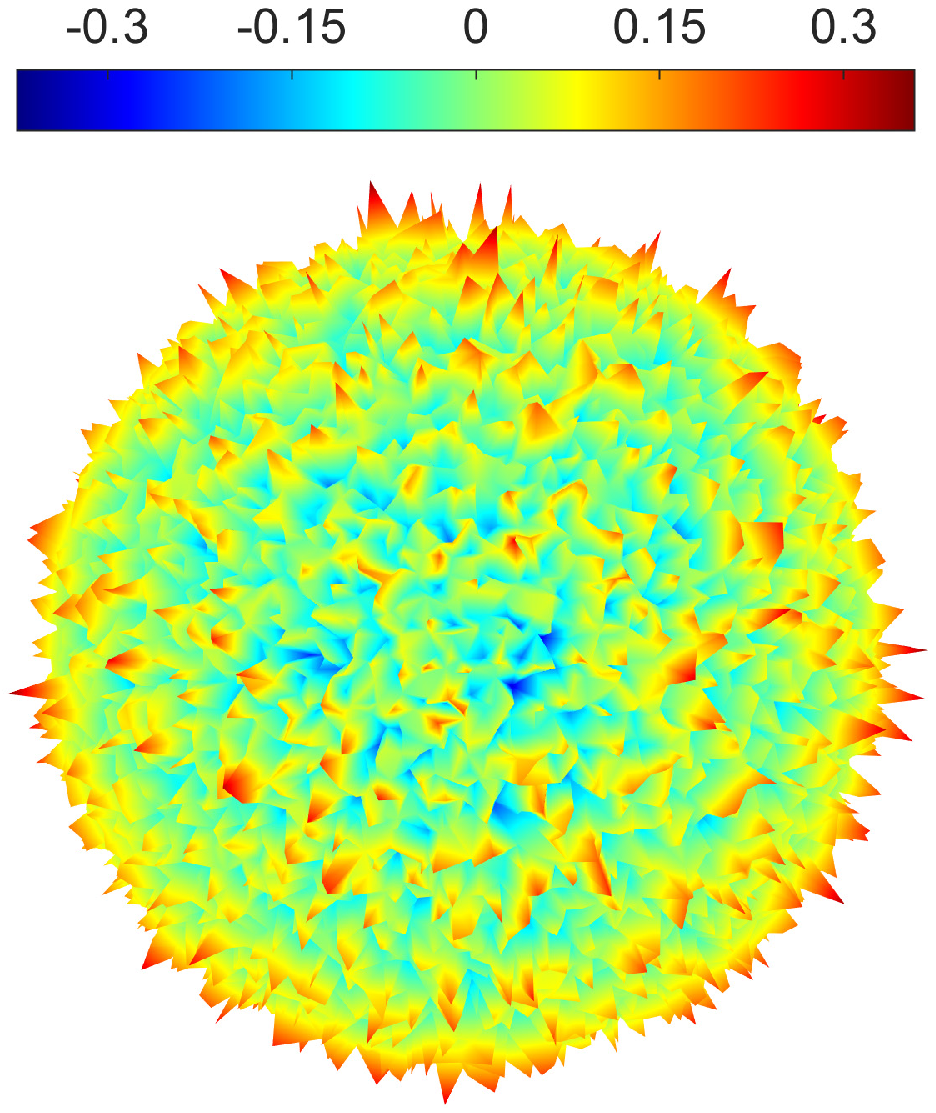}}
    \subfigure[DFH]{\includegraphics[width=2.15cm,height=2.5cm]{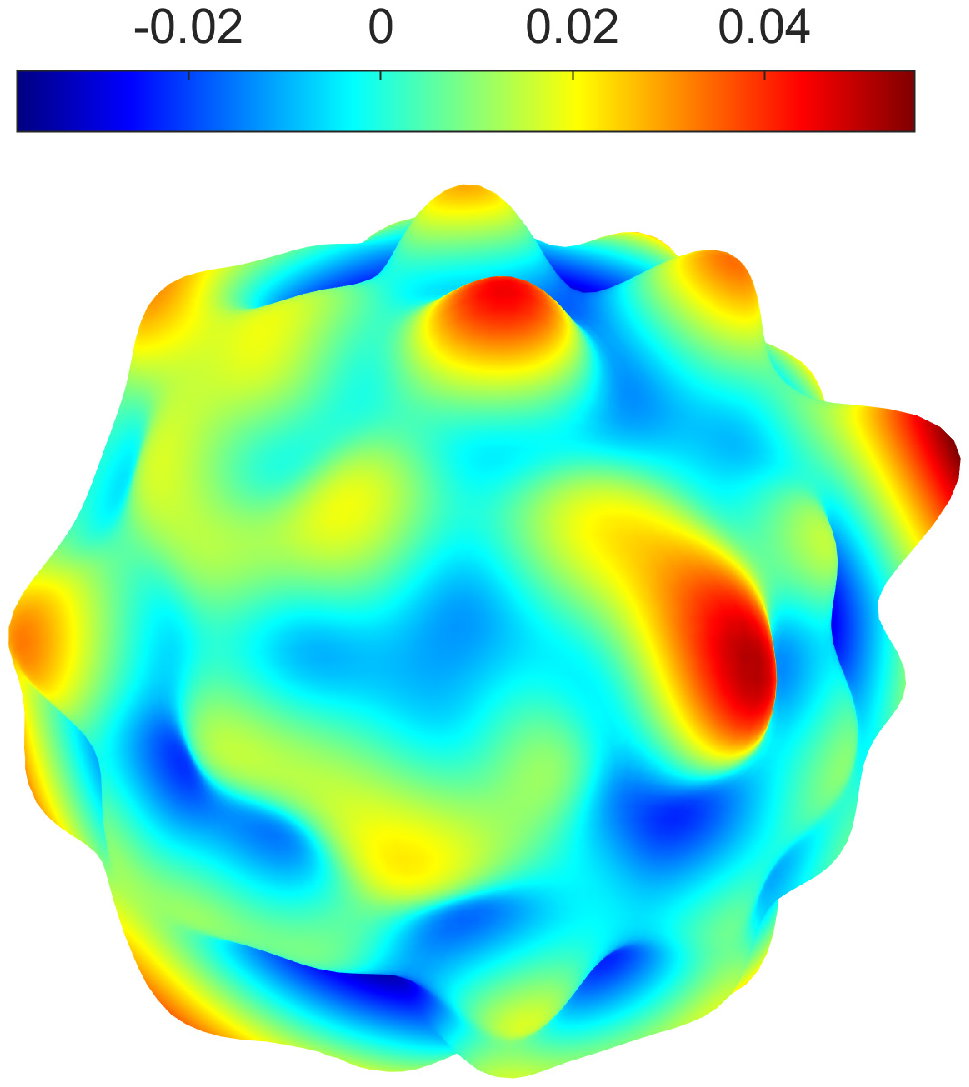}}
    \subfigure[DWRLS(10)]{\includegraphics[width=2.15cm,height=2.5cm]{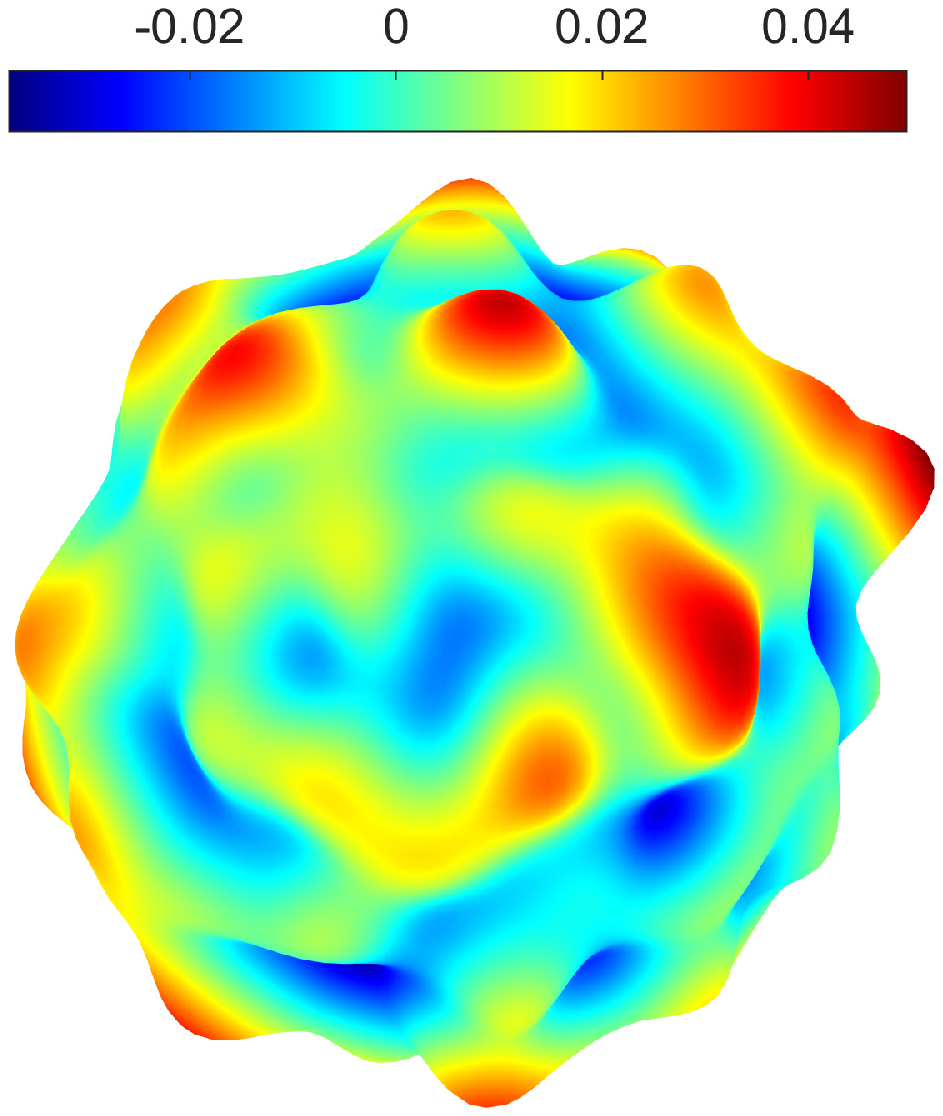}}
    \subfigure[DWRLS(50)]{\includegraphics[width=2.15cm,height=2.5cm]{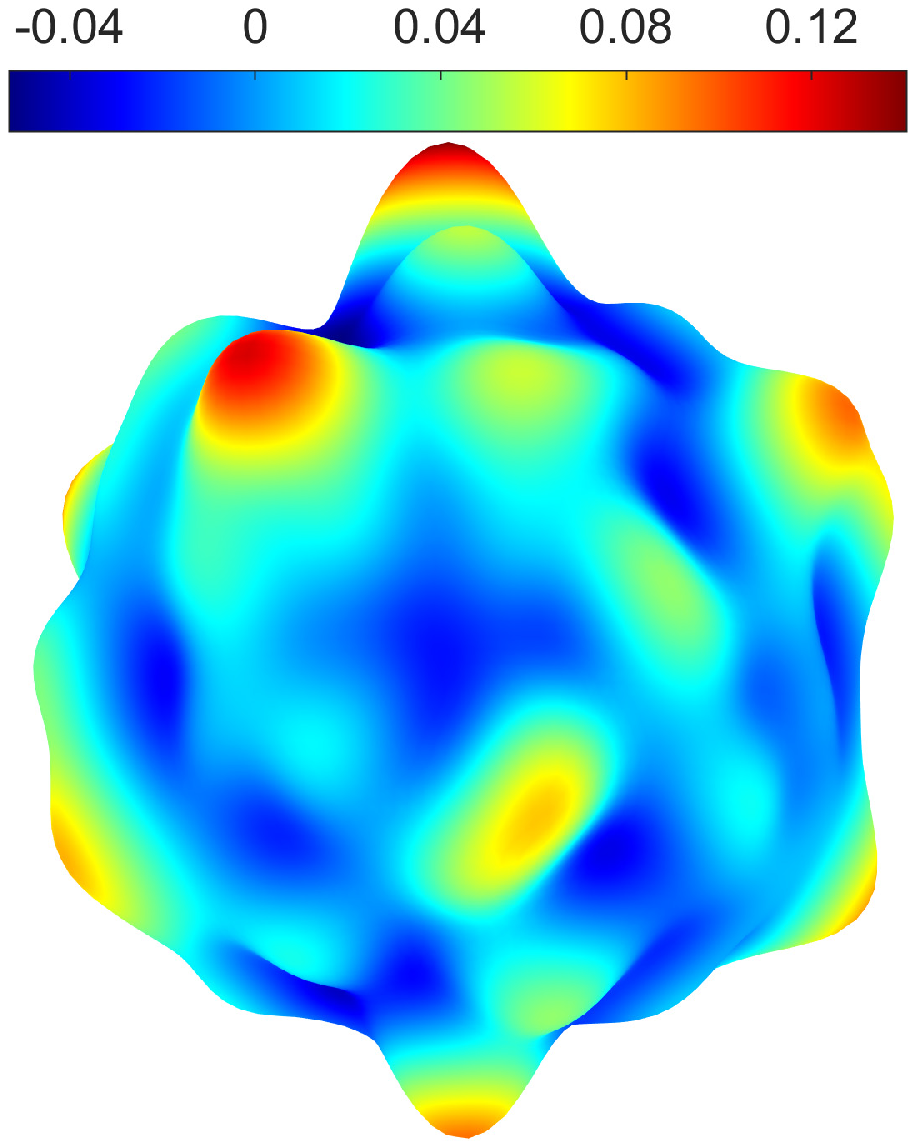}}
    \subfigure[DWRLS(100)]{\includegraphics[width=2.15cm,height=2.5cm]{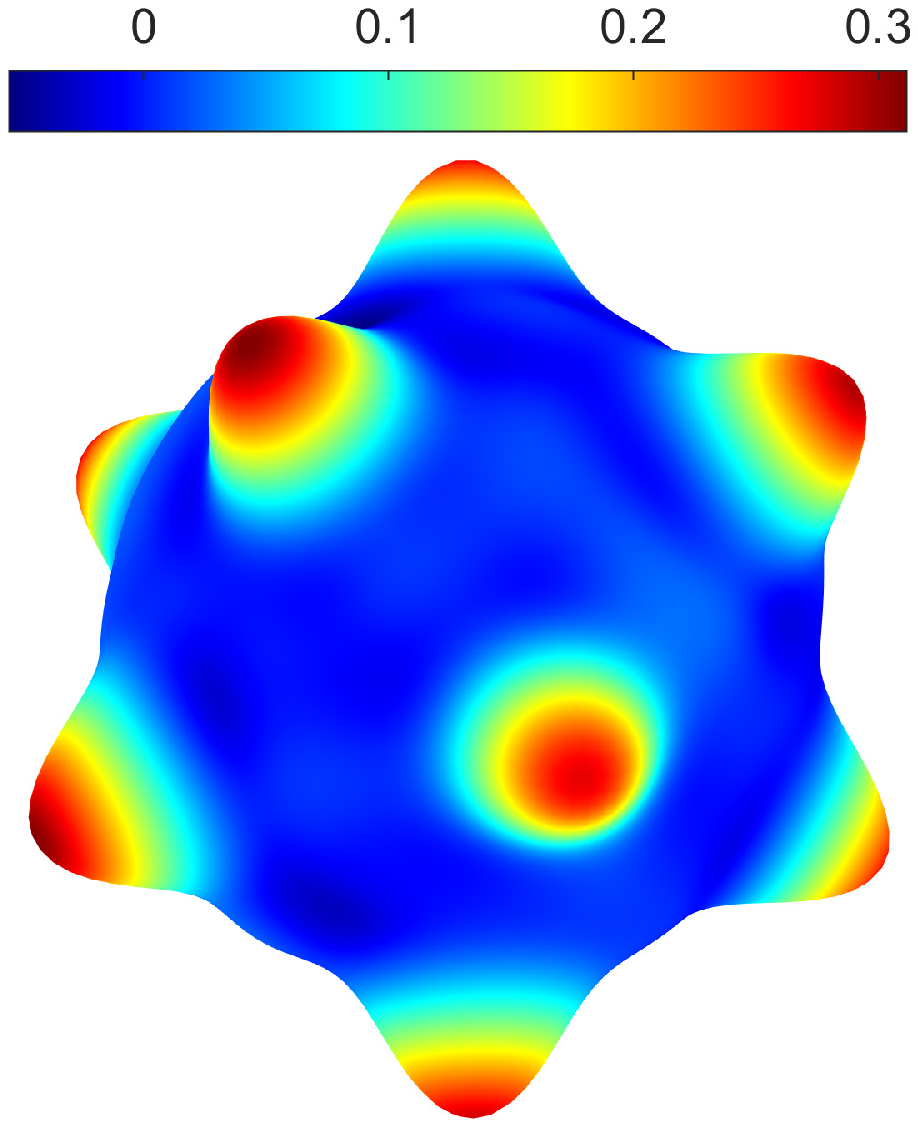}}
	\caption{Comparison on visualization of approximation results of $f_1$}\label{Comparisonf1}
\end{figure*}

\begin{figure*}[t]
    \centering
    \subfigcapskip=-1pt
    \subfigure{\includegraphics[width=2.15cm,height=2.5cm]{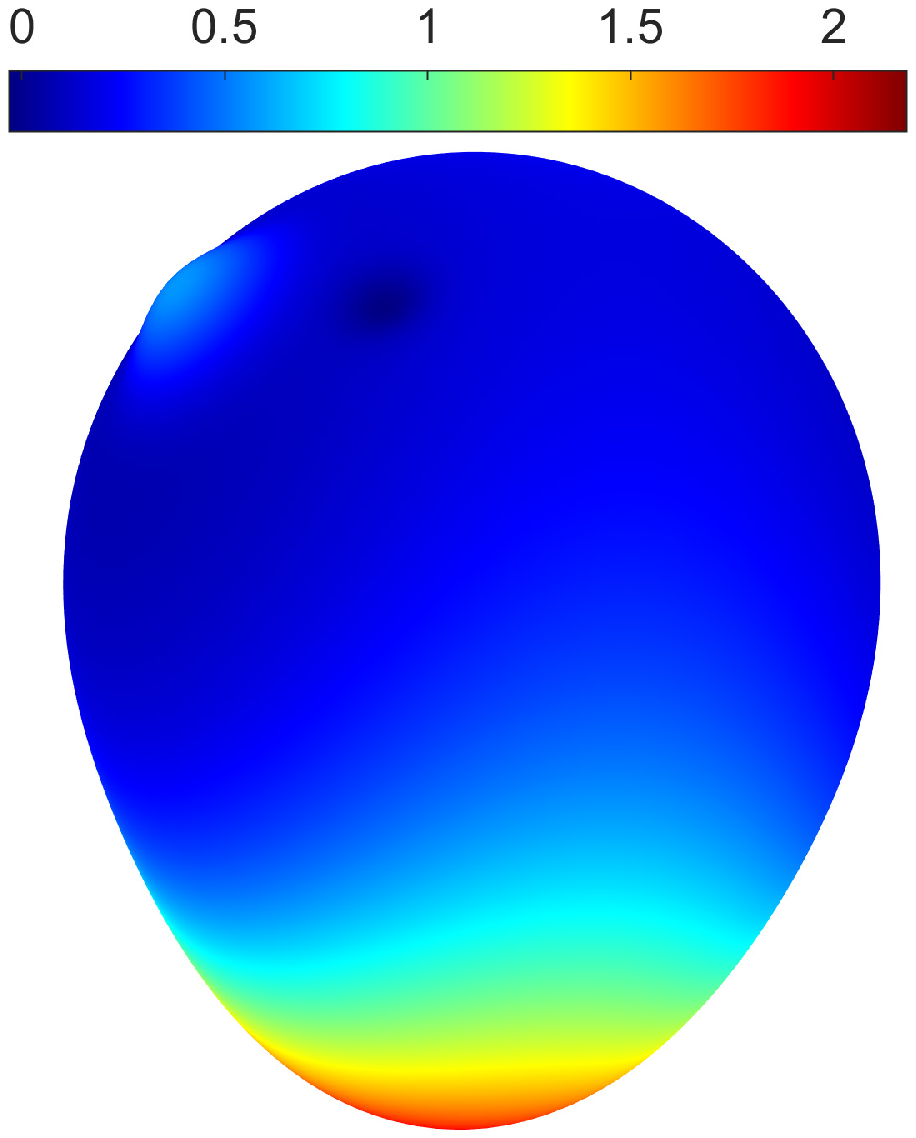}}
    \subfigure{\includegraphics[width=2.15cm,height=2.5cm]{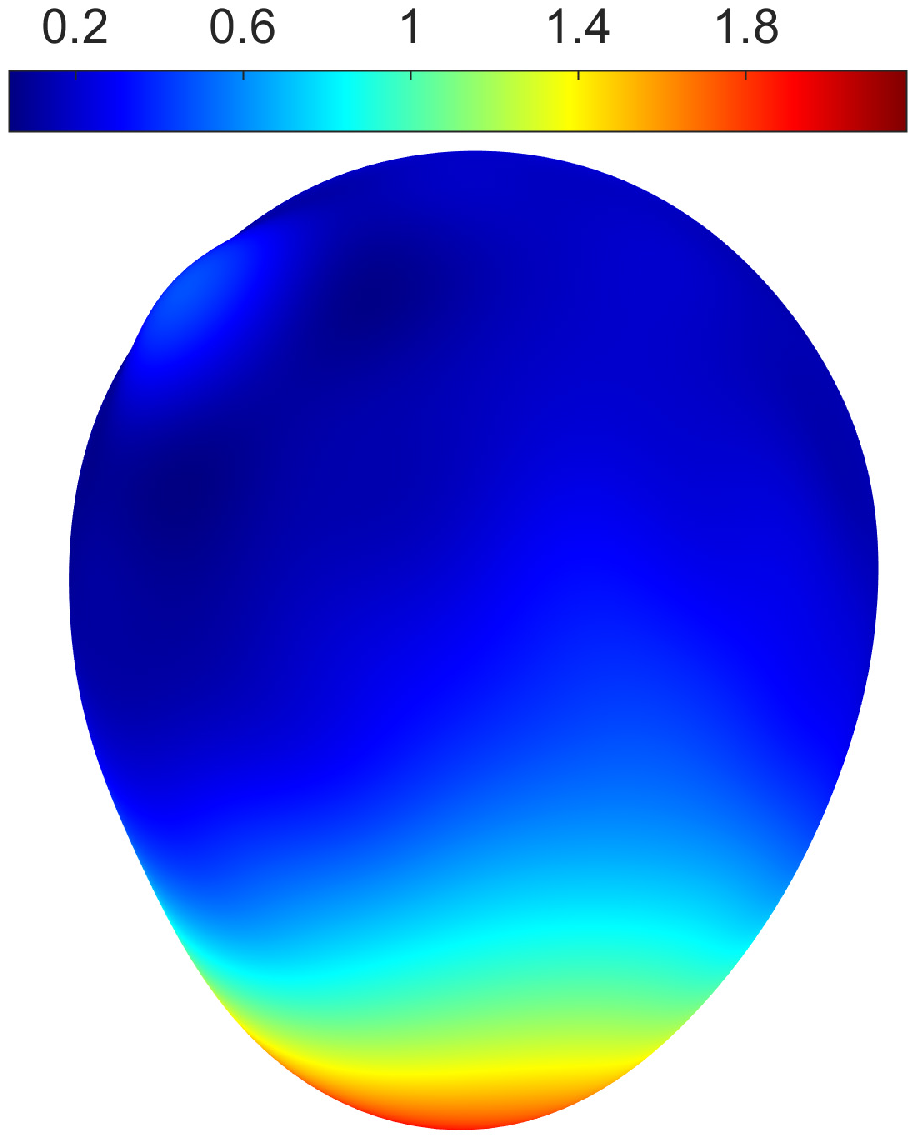}}
    \subfigure{\includegraphics[width=2.15cm,height=2.5cm]{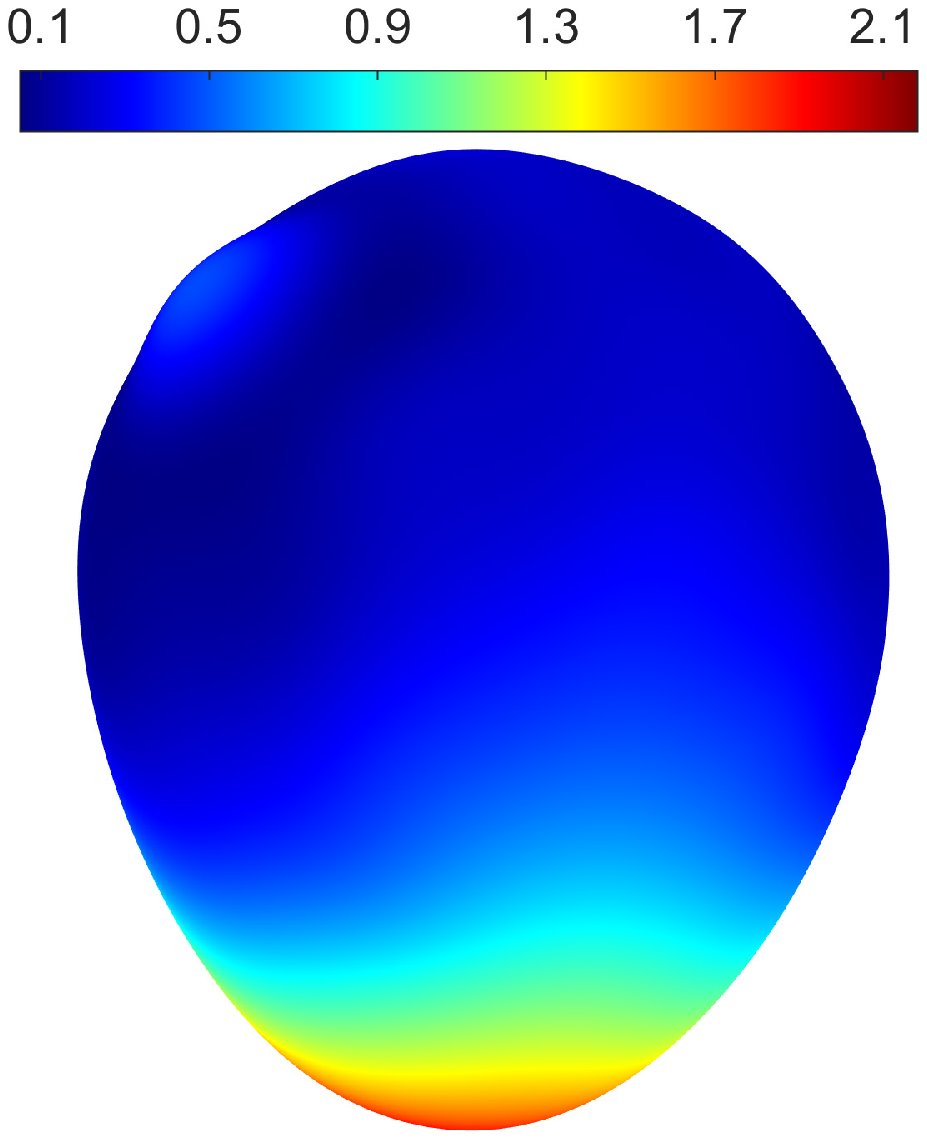}}
    \subfigure{\includegraphics[width=2.15cm,height=2.5cm]{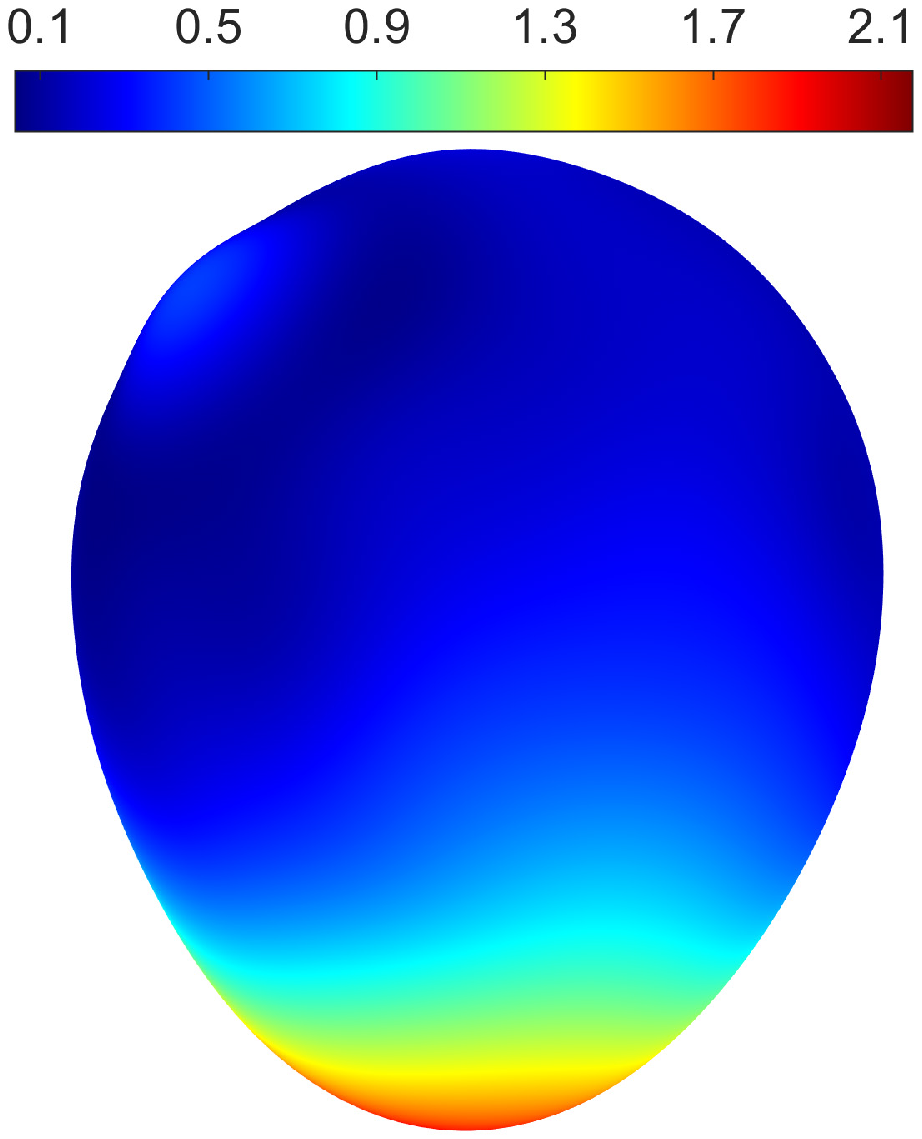}}
    \subfigure{\includegraphics[width=2.15cm,height=2.5cm]{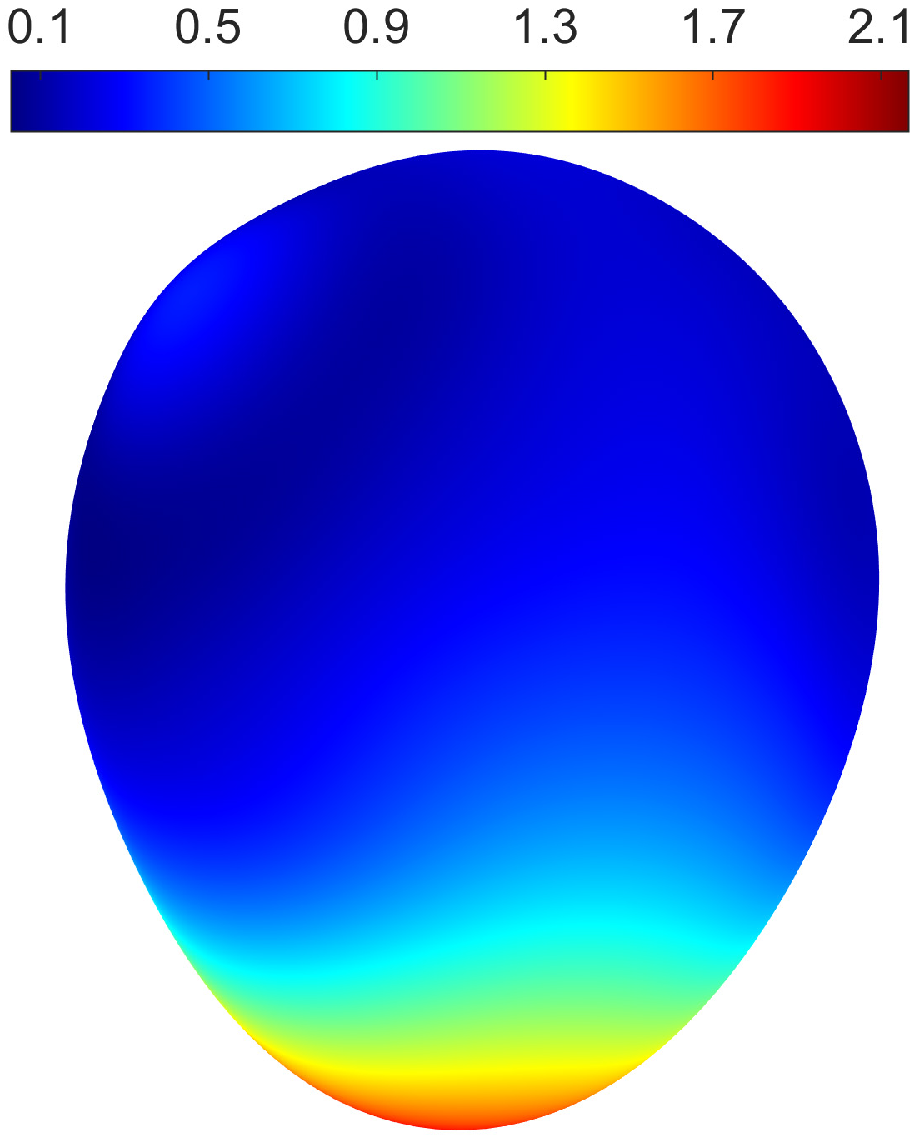}}\\
    \vspace{-0.08in}
    \setcounter{subfigure}{0}
    \subfigure[Groundtruth]{\includegraphics[width=2.15cm,height=2.5cm]{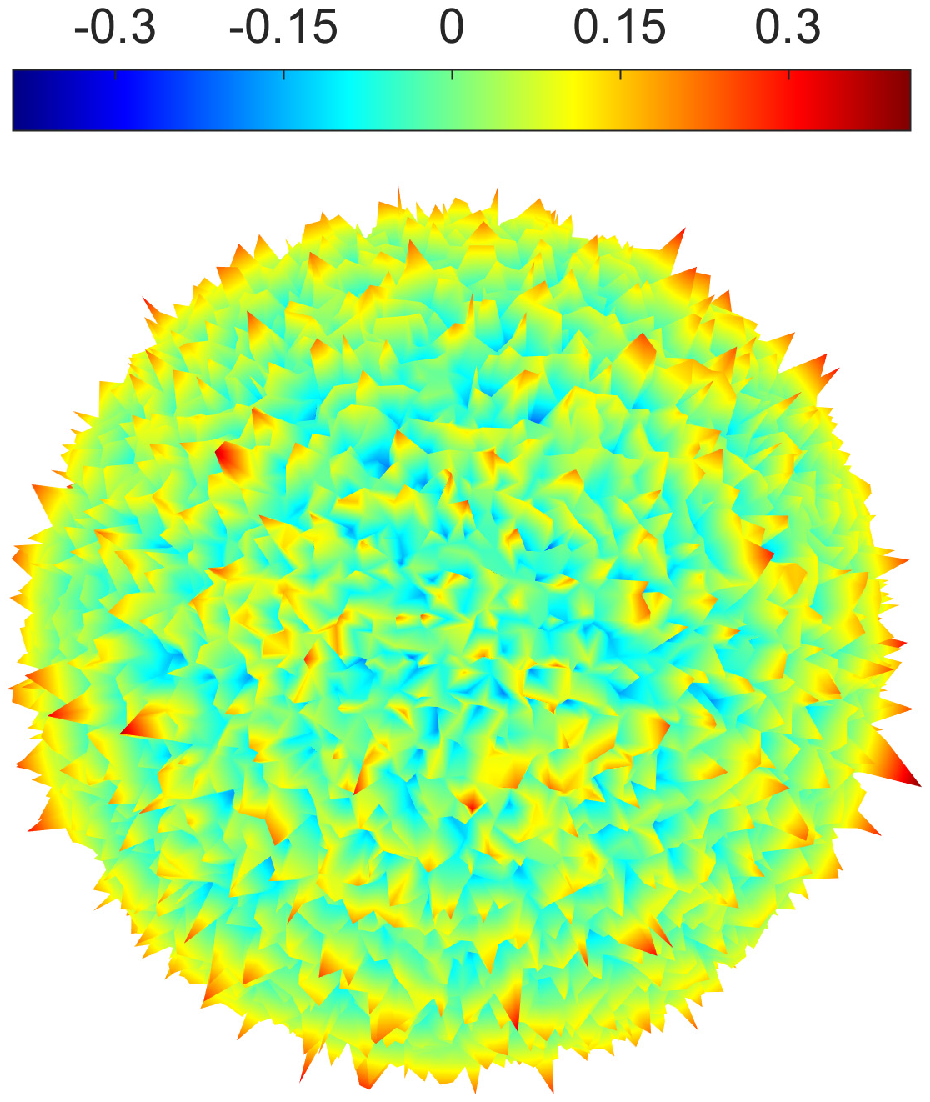}}
    \subfigure[DFH]{\includegraphics[width=2.15cm,height=2.5cm]{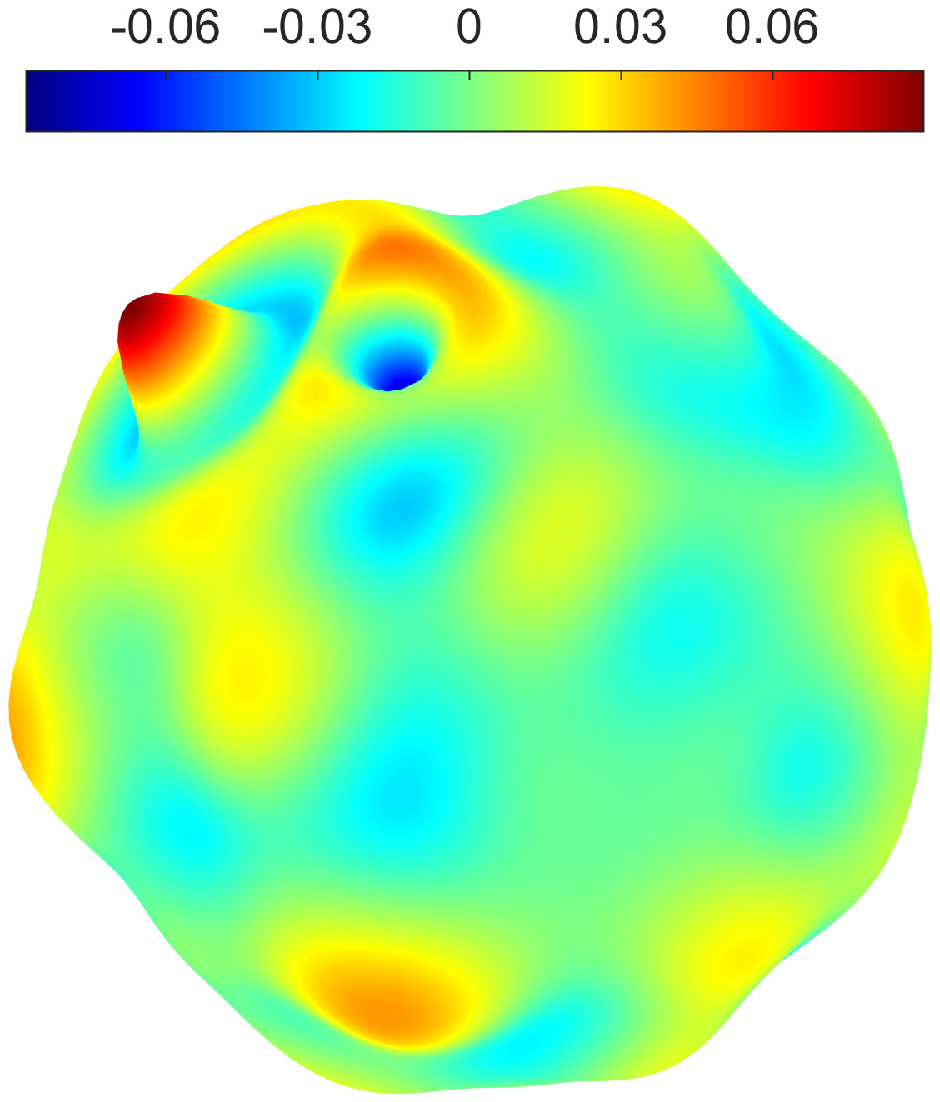}}
    \subfigure[DWRLS(10)]{\includegraphics[width=2.15cm,height=2.5cm]{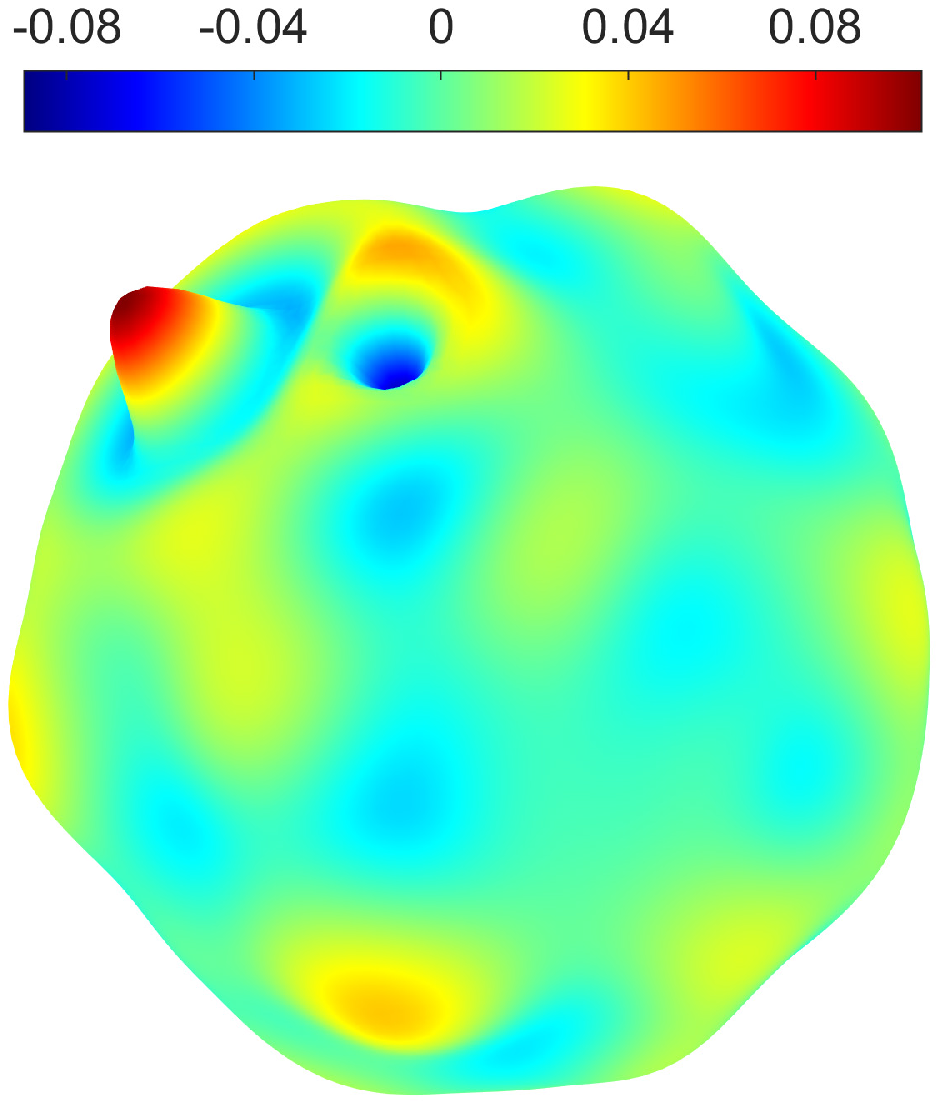}}
    \subfigure[DWRLS(50)]{\includegraphics[width=2.15cm,height=2.5cm]{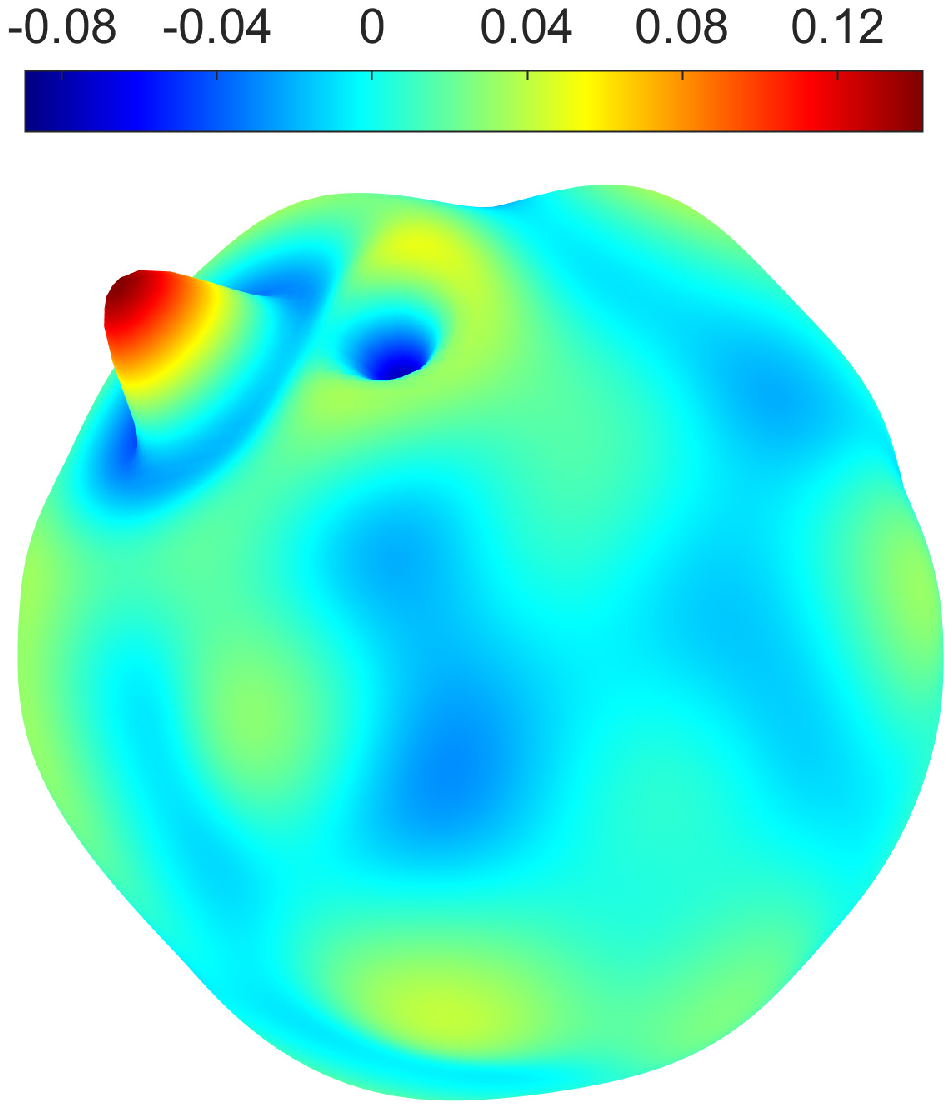}}
    \subfigure[DWRLS(100)]{\includegraphics[width=2.15cm,height=2.5cm]{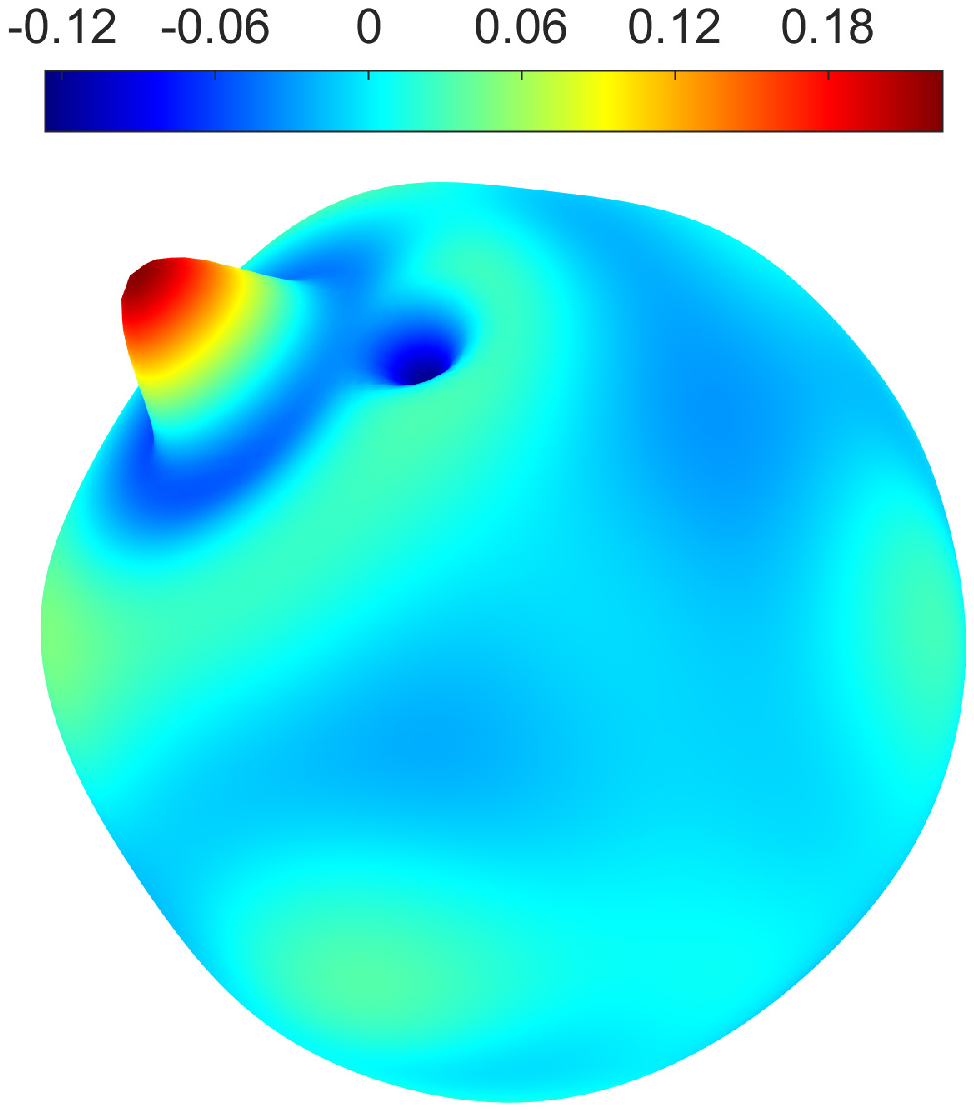}}
	\caption{Comparison on visualization of approximation results of $f_2$}\label{Comparisonf2}
\end{figure*}

\section*{Appendix: Proofs of Proposition \ref{Proposition:3} and Proposition \ref{Proposition:operator-KRR}}

To prove Proposition \ref{Proposition:3}, we need the following spherical Marcinkiewicz-Zygmund inequalities for spherical polynomials \cite{Dai2006,Dai2006a}.
\begin{lemma}\label{Lemma:1}
Let $\mathcal Q_{\Lambda,s}:=\{(w_{i,s},  x_i): w_{i,s}\geq 0
\hbox{~and~}   x_i\in \Lambda\}$  be  a positive
 quadrature rule   on $\mathbb S^d$ with degree $s\in\mathbb N$.  For   any  $P\in \Pi_{s'}^d$ with $s'\in\mathbb N$, there holds
\begin{eqnarray*}
   \sum_{x_i\in\Lambda}w_{i,s} |P(x_i)|^2  &\leq& \tilde{c}_1(s'/s)^{d} \|P\|^2_{L^2(\mathbb S^d)},\qquad s'>s,\\
 \tilde{c}_2\|P\|_{L^2(\mathbb S^d)}^2 &\leq& \sum_{x_i\in\Lambda}w_{i,s} |P(x_i)|^2 \leq \tilde{c}_3\|P\|_{L^2(\mathbb S^d)}^2,\qquad  s'\leq s,
\end{eqnarray*}
where $\tilde{c}_1,\tilde{c}_2,\tilde{c}_3$ are constants depending only on $d$.
\end{lemma}

\begin{proof}[Proof of Proposition \ref{Proposition:3}]
For any $h\in\mathcal N_\phi$, we get $\eta_{\lambda,u}( L_\phi)h\in\mathcal N_\phi$.
Set  $h_{-1,\lambda,u}(x)=0$ and
$$
   h_{j,\lambda,u}(x):=\sum_{k=0}^{2^j}\eta_{\lambda,u}(\hat{\phi}_k)\sum_{\ell=1}^{Z(d,k)}\hat{h}_{k,\ell}Y_{k,\ell}(x).
$$
We have from $\eta_{\lambda,u}(t)=(\lambda+t)^{-u}$ and  $\hat \phi_k\sim k^{-2\gamma}$ with $\gamma> d/2$  that
\begin{equation}\label{convergence}
    \eta_{\lambda,u}( L_\phi)h=\sum_{j=0}^\infty h_{j,\lambda,u}(x)-h_{j-1,\lambda,u}(x)
\end{equation}
 and
\begin{eqnarray}\label{Jackson-for-h}
    \|h_{j,\lambda,u}-h_{j-1,\lambda,u}\|_{L^2(\mathbb S^d)}
    &\leq&
     \|h_{j,\lambda,u}-\eta_{\lambda,u}( L_\phi)h\|_{L^2(\mathbb S^d)}
      +
      \|h_{j,\lambda,u}-\eta_{\lambda,u}( L_\phi)h\|_{L^2(\mathbb S^d)}   \nonumber\\
      &\leq&
      2\left( \sum_{k=2^{j-1}+1}^\infty \frac{\hat{\phi}_k}{(\hat{\phi}_k+\lambda)^u}(\hat{\phi_k})^{-1}\sum_{\ell=1}^{Z(d,k)}(\hat{h}_{k,\ell})^2\right)^{1/2}\nonumber\\
      &\leq&
      \tilde{c}_4 \frac{2^{-\gamma j}}{(2^{-2\gamma j}+\lambda)^u}\|h\|_\phi,
\end{eqnarray}
where $\tilde{c}_4$ is a constant depending only on $\gamma$. For any $f,g\in\mathcal N_\phi$ and $u,v\in[0,1]$, we get from
  \eqref{convergence} that
\begin{eqnarray*}
    (\eta_{\lambda,u}( L_\phi)f)(x)(\eta_{\lambda,v}(L_\phi) g)(x)&=&
    \sum_{j+ \ell\leq L}(f_{j,\lambda,u}-f_{j-1,\lambda,u})(g_{\ell,\lambda,v}-g_{\ell-1,\lambda,v})\\
    &+&\sum_{j+ \ell> L}(f_{j,\lambda,u}-f_{j-1,\lambda,u})(g_{\ell,\lambda,v}-g_{\ell-1,\lambda,v}),
\end{eqnarray*}
where
  $L$ is the unique integer satisfying
\begin{equation}\label{Def.L}
      2^L\leq s<2^{L+1}
\end{equation}
and $f_{j,\lambda,u}$, $g_{\ell,\lambda,v}$ are similar as $h_{j,\lambda,u}$ with $h$ being replacing by $f$, $g$, respectively.
Denote
\begin{eqnarray*}
      \mathcal A_{j,\ell,\lambda,u,v}
      &:=&
      \left| \int_{\mathbb S^d}\left(f_{j,\lambda,u}(x)-f_{j-1,\lambda,u}(x)\right)\left(g_{\ell,\lambda,v}(x)-g_{\ell-1,\lambda,v}(x)\right) d \omega(x)\right.\\
      &-& \left.\sum_{x_i\in \Xi} w_{i,s}\left(f_{j,\lambda,u}(x_i)-f_{j-1,\lambda,u}(x_i)\right)\left(g_{\ell,\lambda,v}(x_i)-g_{\ell-1,\lambda,v}(x_i)\right)\right|.
\end{eqnarray*}
Since $\mathcal Q_{\Xi,s} =\{(w_{i,s},  x_i): w_{i,s}> 0
\hbox{~and~}   x_i\in \Xi\}$  is  a positive
 quadrature rule   on $\mathbb S^d$ with degree $s\in\mathbb N$, we have
 $$
 \mathcal A_{j,\ell,\lambda,u,v}=0,\qquad \forall 2^{j+ \ell}\leq 2^L\leq s.
$$
Hence,
\begin{eqnarray*}
  &&\left|\int_{\mathbb{S}^{d}} (\eta_{\lambda,u}(  L_\phi)f)(x)(\eta_{\lambda,v}(  L_\phi)g)  (x) d \omega(x)-\sum_{x_i\in\Xi} w_{i, s} (\eta_{\lambda,u}(  L_\phi)f)(x_{i})(\eta_{\lambda,v}(\mathcal L_\phi)g)(x_i)\right|\nonumber\\
  &\leq&  \sum_{j+ \ell >  L}  \mathcal A_{j,\ell,\lambda,u,v}
   =  \left(\sum_{j+ \ell> L, j, \ell\leq  L}+\sum_{  j> L, \ell\leq  L}+\sum_{ j\leq  L, \ell>  L}+\sum_{ j , \ell\geq  L}\right) \mathcal A_{j,\ell,\lambda,u,v} \nonumber\\
  &=:&I_1+I_2+I_3+I_4.
\end{eqnarray*}
But the H\"{o}lder inequality implies
\begin{eqnarray*}
    && \mathcal A_{j,\ell,\lambda,u,v} \leq \|f_{j,\lambda,u}-f_{j-1,\lambda,u}\|_{L^2(\mathbb S^d)} \|g_{\ell,\lambda,v}-g_{\ell-1,\lambda,v}\|_{L^2(\mathbb S^d)}\\
    &+&
    \left(\sum_{x_i\in \Lambda}w_{i,s}|(f_{j,\lambda,u}(x_i)-f_{j-1,\lambda,u}(x_i)|^2\right)^{1/2}\left(\sum_{x_i\in \Lambda}w_{i,s}|(g_{\ell,\lambda,v}(x_i)-g_{\ell-1,\lambda,v}(x_i)|^2\right)^{1/2}.
\end{eqnarray*}
Then we get from Lemma \ref{Lemma:1} and \eqref{Jackson-for-h} that
\begin{eqnarray*}
     &&\mathcal A_{j,\ell,\lambda,u,v}
     \leq
    \tilde{c}_1^2 \frac{2^{-\gamma (j+\ell)}}{ (2^{-2\gamma j}+\lambda)^u(2^{-2\gamma \ell}+\lambda)^v}  \|f\|_\phi\|g\|_\phi\\
     &+&
     \tilde{c}_2 \left\{\begin{array}{cc}
        \frac{2^{-\gamma (j+\ell)}}{(2^{-2\gamma j}+\lambda)^u(2^{-2\gamma \ell}+\lambda)^v} \|f\|_\phi\|g\|_\phi,& \mbox{if}\ j+ \ell > L, j\leq L, \ell\leq L,\\
     2^{(j-L)d/2}  \frac{2^{-\gamma (j+\ell)}}{(2^{-2\gamma j}+\lambda)^u(2^{-2\gamma \ell}+\lambda)^v} \|f\|_\phi\|g\|_\phi,& \mbox{if}\ j>L, \ell\leq L,\\
      2^{( \ell-L)d/2} \frac{2^{-\gamma (j+\ell)}}{(2^{-2\gamma j}+\lambda)^u(2^{-2\gamma \ell}+\lambda)^v} \|f\|_\phi\|g\|_\phi,&\mbox{if}\ j\leq L, \ell>L,\\
      2^{(j+ \ell-2L)d/2} \frac{2^{-\gamma (j+\ell)}}{(2^{-2\gamma j}+\lambda)^u(2^{-2\gamma \ell}+\lambda)^v}\|f\|_\phi\|g\|_\phi,&\mbox{if}\ j> L,\ell> L,
     \end{array}
     \right.
\end{eqnarray*}
where $\tilde{c}_1,\tilde{c}_2$ are constants depending only on $d,\gamma$. Since $u+v\leq 1$, we have for any $\lambda>0$ that
\begin{eqnarray*}
   &&I_1= \sum_{j+ \ell>L, j, \ell\leq L} \mathcal A_{j,\ell,u,v} \leq
       \tilde{c}_3
       \sum_{j+ \ell>L, j, \ell\leq L}2^{ -(j+\ell)(1-u-v)\gamma}\|f\|_\phi\|g\|_\phi\\
      &\leq& \tilde{c}_3  \|f\|_\phi\|g\|_\phi
      \sum_{k=L+1}^{2L}2^{ -k(1-u-v)}
       \leq
      \tilde{c}_{4,1}\|f\|_\phi\|g\|_\phi
      \left\{\begin{array}{cc}
        2^{-L(1-u-v)\gamma},   & u+v<1 \\
        L   & u+v=1,
      \end{array}
      \right.
\end{eqnarray*}
where $c_3$ and $\bar{c}_{4,1}$ are constants depending only on $u,v,\gamma,d$.
Furthermore, for any $u,v\in[0,1)$ satisfying $u+v\leq 1$,   and any $\lambda\geq 2^{-2\gamma L}$, we have   from $\gamma>d/2$ that
\begin{eqnarray*}
     I_2
           &\leq&
          \tilde{c}_{4,2}\|f\|_\phi\|g\|_\phi\lambda^{-u}2^{-\gamma L}, \\
           I_3
      & \leq&
         \tilde{c}_{4,3}\|f\|_\phi\|g\|_\phi\lambda^{-v}2^{-\gamma L},\\
         I_4 &\leq&
         \tilde{c}_{4,4}\|f\|_\phi\|g\|_\phi
         \lambda^{-(u+v)}2^{-2\gamma L},
\end{eqnarray*}
where $\tilde{c}_{4,j}$ for $j=2,3,4$ are constants depending only on $\gamma,d,u,v$.
   Combining all the above estimate, we obtain the desired estimate with $\tilde{c}:= \bar{c}_{3,1}+\bar{c}_{3,2}+\bar{c}_{3,3}+\bar{c}_{3,4}$.
This completes the proof of Proposition \ref{Proposition:3}.
\end{proof}

\begin{proof}[Proof of Proposition \ref{Proposition:operator-KRR}]
Since the functional derivative \cite{Smale2005} of (\ref{WRLS}) is
$$
     2\sum_{i=1}^{|D|} w_{i,s}(f(x_i)-y_i)\phi_{x_i}+2\lambda f,
$$
we get
$$
      \sum_{i=1}^{|D|} w_{i,s}(y_i-f_{D,W_s,\lambda}(x_i))\phi_{x_i}=\lambda f_{D,W_s,\lambda}.
$$
This together with (\ref{def.empirical-operator}) and (\ref{def.adjoint-operator}) implies
$$
         (L_{\phi,D,W_s}+\lambda I)f_{D,W_s,\lambda}=S_{D,W_s}^Ty_D
$$
Noting further $L_{\phi,D,W_s}$ is a positive operator, we then have
$$
      f_{D,W_s,\lambda}=(L_{\phi,D,W_s}+\lambda I)^{-1}S_{D,W_s}^Ty_D.
$$
This completes the proof of Proposition \ref{Proposition:operator-KRR}.
\end{proof}


\section*{Acknowledge} {The authors would like to thank two anonymous referees for their constructive suggestions and Professor Di Wang from X'an Jiaotong University  for his fruitful suggestions on the numerical simulations.  The research was supported partially by the  National Key R\&D Program of China (No.2020YFA0713900).
 The work of H. Feng is  supported partially  by the Research Grants Council of Hong Kong [Project  \# CityU 11306620].
The work  of S. B. Lin is supported partially by
   the National Natural Science Foundation of China
(Nos.62276209,61876133). The first version of the paper was written when D. X. Zhou worked at City University of Hong Kong,
   supported partially by the Research Grants Council of Hong Kong [Project \# CityU 11308020, N-CityU102/20, C1013-21GF], Hong Kong Institute for Data Science,  Germany/Hong Kong Joint Research Scheme  [Project No. G-CityU101/20],
 Hong Kong Institute for Data Science, Laboratory for AI-Powered Financial Technologies, and National Science Foundation of China [Project No. 12061160462].

\end{document}